%% file: main.tex
\documentclass[11pt]{article} 
\pdfoutput=1

\RequirePackage[margin=1in]{geometry}

\usepackage[utf8]{inputenc}
\usepackage{amssymb}
\usepackage{hyperref}
\usepackage{amsthm}
\usepackage{amsmath}
\usepackage{todonotes}
\usepackage{soul}
\usepackage{graphicx}	
\usepackage{float}
\usepackage{algorithm}
\usepackage[noend]{algorithmic}
\usepackage{color}
\usepackage{enumitem}
\usepackage{caption}
\usepackage{thmtools}
\usepackage{thm-restate}
\usepackage{hyperref}
\usepackage{cleveref}
\setlist[1]{itemsep=-5pt}

\newtheorem{theorem}{Theorem}

\newtheorem{lemma}[theorem]{Lemma}
\newtheorem{definition}{Definition}

\newtheorem{assumption}{Problem Setup}
\newtheorem{assumption2}{Assumption}

\newcommand{\citealp}[1]{\cite{#1}}
\newcommand{\citep}[1]{\cite{#1}}
\newcommand{\citet}[1]{\cite{#1}}

\newcommand{\bigo}[1]{O \left( {#1} \right)}
\newcommand{\logo}[1]{\tilde{O} \left( {#1} \right)}

\newcommand{\bigomega}[1]{\Omega \left( {#1} \right)}
\newcommand{\bigtheta}[1]{\Theta \left( {#1} \right)}

\newcommand{\remove}[1]{}

\DeclareMathOperator{\gain}{{\texttt{Gain}}}
\newcommand{\sample}{\mathcal{S}}
\newcommand{\F}{\mathcal{F}}
\newcommand{\Z}{\mathcal{Z}}

\newcommand{\setG}{\mathcal{G}}

\newcommand{\tildeF}{{\tilde{\F}}}
\newcommand{\I}{\mathcal{I}}

\newcommand{\distr}{\mathcal{D}}

\newcommand{\N}{N}
\newcommand{\K}{K}

\newcommand{\samplesize}{{S}} 

\newcommand{\LFD}{\textsf{\textsc{UseRep}}} %{\textsc{LearnUsingRepresentation}}
\newcommand{\ID}{\textsf{\textsc{ImproveRep}}}

\newcommand{\nodeA}{u}
\newcommand{\nodeB}{v}
\newcommand{\nodeC}{w}

\newcommand{\Var}{\textnormal{\texttt{var}}}
\newcommand{\Fold}{{\tilde{\mathcal{F}}_{\text{old}}}}

\DeclareMathOperator{\dlin}{ {\in_{\texttt{dl}}} }
\DeclareMathOperator{\dlnotin}{ {\notin_{\texttt{dl}}} }
\DeclareMathOperator{\dlsubset}{{\subset_{\texttt{dl}}}}

\DeclareMathOperator{\dlsubseteq}{{\subseteq_{\texttt{dl}}}}

\newcommand{\dtspace}{\mathbb{DT}}

\newcommand{\x}{\mathbf{x}}

\newcommand{\G}{\mathbf{g}}
\newcommand{\g}{g}

\newcommand{\tildeG}{\tilde{\G}}
\newcommand{\Pref}{\textnormal{\textsf{Pref}}}
\newcommand{\sPref}{\textnormal{\textsf{Pref}}_{\star}}
\newcommand{\sSuff}{\textnormal{\textsf{Suff}}_{\star}}
\newcommand{\ssSuff}{\textnormal{\textsf{Suff}}_{\star\star}}
\newcommand{\f}{\mathbf{f}}
\newcommand{\w}{\mathbf{w}}
\newcommand{\E}{\mathbb{E}}

\newcommand{\C}{\mathbb{C}}

\newcommand{\affix}{\textsf{\textsc{affix}}}
\newcommand{\lab}{\textsf{\textsc{label}}}
\newcommand{\conflict}{\textsf{\textsc{conflict}}}
\newcommand{\induce}{\textsf{\textsc{induce}}}

\usepackage{times}

\allowdisplaybreaks

\begin{document}

\title{Lifelong Learning in Costly Feature Spaces\footnote{Authors' addresses: \texttt{\{ninamf,avrim,vaishnavh\}@cs.cmu.edu}.}}
 \author{Maria-Florina Balcan \and
 Avrim Blum \and
 Vaishnavh Nagarajan }
\date{\today}

\sloppy 
\maketitle
\input{abstract.tex}

\section{Introduction}
\input{small-intro.tex}

\section{Decision Trees}
\label{sec:dt}
\input{decision-trees.tex}

\section{Monomials}
\label{sec:monomials}
\input{monomials.tex}
\input{polynomials.tex}

\section{The Agnostic Case}
\label{sec:agnostic}
\input{agnostic.tex}

\section{Discussion and Open Problems}

 Lifelong learning is an important goal of modern machine learning systems that has largely been studied only empirically. In this work,  we  theoretically analyze lifelong learning from the perspective of feature-efficiency. More specifically, we show how, when a series of tasks are related through metafeatures, knowledge can be extracted from previously-learned tasks and stored in a succinct representation in order to learn future tasks by examining only few relevant features on the training datapoints. To this end, we present feature-efficient lifelong learning algorithms with guarantees for widely studied classes of targets, namely, decision trees, decision lists and real-valued monomials and polynomials. We also present algorithms for an agnostic scenario where some of the targets may be adversarially unrelated to the other targets.  Finally, we derive lower bounds on the feature-efficiency of a lifelong learner in this model, which show that under some conditions, the guarantees of our algorithms are tight.\\

An open technical question is whether our lower bounds can be extended to incorporate problem-specific parameters such as the depth of a tree/list or the degree of a monomial/polynomial. In particular, while the feature-efficiency bound for our decision tree learning algorithm has a dependence of $Kd$, it is not clear whether a bound of $K+d$ is achievable. Another open question is whether it is possible to characterize the hardness of recovering the metafeatures exactly in the case of decision trees and lists (even though our algorithms work without having to recover the metafeatures exactly). Finally, we note that as a high level direction for theoretical research in lifelong learning, it would be interesting to explore different ways of formalizing task relations for various families of targets, and to explore the different kinds of resource-efficiency bounds they can guarantee, while also understanding their limitations. \\ 

\noindent
\textbf{Acknowledgements}. This work was supported in part by the National Science Foundation
   under grants CCF-1535967, CCF-1525971, CCF-1422910, IIS-1618714, a
   Sloan Research Fellowship, a Microsoft Faculty Fellowship, and a
   Google Research Award.

\bibliographystyle{plain}
\bibliography{bibliography}

\appendix

\section{Decision Trees}
\label{app:dt}

\input{dt-appendix.tex}

\section{Monomials}
\label{app:pm}
\input{pm-appendix.tex}

\subsection{Polynomials}
\label{app:poly}

\input{poly-appendix.tex}

\end{document}

%% file: abstract.tex
\begin{abstract}
An important long-term goal in machine learning systems is to build learning agents that, like humans, can learn many tasks over their lifetime,
and moreover use information from these tasks to improve their ability to do so efficiently.
% while also avoiding redundant usage of precious resources. 
%Such a {\em lifelong learner} learns a series of related target functions by a) extracting knowledge about commonalities between previously learned target functions and then b) using this knowledge to more efficiently learn related targets in the future.
 In this work, our goal is to provide new theoretical insights into the potential of this paradigm. In particular, we propose a lifelong learning framework that adheres to a novel notion of resource efficiency that is critical in many real-world domains where feature evaluations are costly. That is, our learner aims to reuse information from previously learned related tasks to learn future tasks in a {\em feature-efficient} manner. 
Furthermore, we consider novel combinatorial ways in which learning tasks can relate.  Specifically, we design lifelong learning algorithms
% in this framework 
for two structurally different and widely used families of target functions: decision trees/lists and monomials/polynomials.
%, both over real-valued inputs. 
%I think this line has a lot of redundant information 
%We present natural relations among such functions and then design algorithms that extract useful information from previous tasks to substantially reduce feature evaluations in future tasks. 
We also provide strong feature-efficiency guarantees for these algorithms; in fact, we show that in order to learn future targets, we need only slightly more feature evaluations per training example than what is needed to predict on an arbitrary example using those targets.
We also provide algorithms with guarantees in an 
%significantly relaxed
 agnostic model where not all the targets are related to each other. Finally, we also provide lower bounds on the performance of a lifelong learner in these models, which are in fact tight under some conditions.

\end{abstract}

%% file: small-intro.tex
Machine learning algorithms have found widespread use in solving naturally occurring tasks in   domains like medical diagnosis, autonomous navigation and document classification. Accompanying this rapid growth, there has been remarkable progress in theoretically understanding how machine learning can solve single tasks in isolation. However, real-world tasks rarely occur in isolation. For example, an autonomous robot may have to accomplish a series of control learning tasks during its life, and to do so well it should employ methods that improve its ability to learn as it does so, needing less resources as it learns more \citep{TP97,ll-robot}. As we scale up our goals from learning a single function to learning a stream of many functions, we need to develop sound theoretical foundations to analyze these large-scale learning settings.

Broadly, the goal of a {\em lifelong learner} is to solve a series of many tasks over its lifetime by a) extracting succinct and useful representations about the relations among previously learned tasks, and then b) using these representations to learn future tasks more efficiently. 
In this work, we provide new insights into this paradigm by first proposing a metric for lifelong learning that exposes an important type of resource efficiency gain. Then we design algorithms for important and widely used classes of functions %(including decision trees, decision lists, and polynomials) 
%with the goal of developing 
%practically useful lifelong learning algorithms 
with strong theoretical guarantees in this metric. 

In particular, we consider a setting where evaluating the features of data points is costly and hence the learner wishes to exploit task relations to improve its {\em feature-efficiency} over time.  
Feature-efficiency is critical in applications such as medical diagnosis and high-dimensional data domains where evaluating feature values of a data point might involve performing expensive or intrusive medical tests or accessing millions of values. 
In fact, one of the reasons decision trees (which is one of the important function classes we study in this paper) are commonly used in medical diagnosis \citep{Podgorelec2002}
is that once the trees are learned, one can then make {\em predictions} on new examples by evaluating very few features---at most the depth of the tree.  

We consider lifelong learning from the perspective of this feature evaluation cost, and show how we can use commonalities among previously-learned target functions to perform much better in learning new related targets according to this cost. Specifically, if we face a stream of $m$ {\em adversarially chosen} related learning tasks over the same set of $N$ features, each with about $\samplesize$ training examples, we will make $\bigo{\samplesize m N}$ feature evaluations if we learn each task from scratch individually. Our goal will be to leverage task relatedness to learn very few tasks from scratch and learn the rest in a feature-efficient manner, making as few as $\bigo{\samplesize(m+N)}$ feature evaluations in total. %, ignoring other problem-specific parameters.  

We study two structurally different classes of 
%practically-relevant 
target functions. %, decision trees and decision lists (Section~\ref{sec:dt}) and  monomials and polynomials (Section~\ref{sec:monomials}), with target functions sharing relations appropriate to those classes.
% and provide feature-efficient lifelong learning algorithms for both. 
In Section~\ref{sec:dt} we focus on decision trees (and lists) which are a widely used class of target functions  \citep{top10dt,RM08,Quinlan86,cart84} popular because of their naturally interpretable structure -- to make a prediction one has to simply make a sequence of feature evaluations -- and their usefulness in the context of prediction in costly feature spaces. In Section~\ref{sec:monomials} we analyze monomial and polynomial functions, an expressive family that can approximate many realistic functions (e.g., Lipschitz functions \citep{andoni}) and is relevant in common machine learning techniques like polynomial regression, curve fitting and basis expansion \citep{hastie}. Our study of polynomials also demonstrates how feature-efficient learning is possible even when the function class is not intrinsically feature-efficient for prediction. The non-linear structure of both of these function classes  poses interesting technical challenges in modeling their relations and proposing feature-efficient solution strategies. 
 Indeed our algorithms will use their learned information to determine an adaptive feature-querying strategy that 
significantly minimizes feature evaluations. %reduces the number of features that must be examined.  
%\todo{rephrase, dictionary learning}

In Section~\ref{sec:dt}, we present our results for decision trees and lists. First, we describe intuitive relations among our targets in terms of a small {\em unknown} set of $K$ ``metafeatures'' or parts of functions common to all targets (think of $K$ much less than $N$). More specifically, these metafeatures are subtrees that can be combined sequentially to represent the target tree. We then present our feature-efficient lifelong learning protocol which involves addressing two key challenges. First, we need a computationally-efficient strategy that can recover useful metafeatures from previously learned targets (Algorithm~\ref{alg:id-dt}). Interestingly, we show that the learned metafeatures can be useful even if they do not exactly match the unknown $K$ metafeatures, so long as they ``contain'' them in an appropriate sense. Second, we need a feature-efficient strategy that can learn new target functions using these learned metafeatures (Algorithm~\ref{alg:lfd-dt}).  Making use of these two powerful routines, we present a lifelong learning protocol that learns only at most $K$ out of $m$ targets from scratch and for the remaining targets examines only $Kd$ features per example (where $d$ is the depth of the targets), thus making $\bigo{S(NK + mKd)}$ feature evaluations in total (Theorem~\ref{thm:dt}).

In Section~\ref{sec:monomials}, we study monomials and polynomials which are similarly related through $K$ unknown metafeatures. We adopt a natural model where the metafeatures are monomials themselves, so that the monomial targets are simply products of metafeatures. In the case of polynomials, this defines a two-level relation, where each polynomial is a sum of products of  metafeatures.  For polynomials, we present an algorithm that learns only $K$ of $m$ targets from scratch and on the remaining targets, evaluate s$\bigo{K+d}$ features per example (where $d$ is the degree of the target), thus making only $\bigo{\samplesize(KN + m(K+d))}$ feature evaluations over all tasks. More interestingly, in the case of large-degree monomials, our algorithm may need fewer feature evaluations per example ($K$) to learn the monomial than that needed ($d$) to evaluate the monomial on an input point. % (Theorem~\ref{thm:monomials}).

Next in Section~\ref{sec:agnostic}, we consider a relaxation of the original model, more specifically, an agnostic case where the learner faces $m+r$ targets, $r$ of which are ``bad'' targets adversarially chosen to be unrelated to the other $m$ interrelated ``good'' targets. As a natural goal, we want the learner to minimize the feature evaluations made on the training data of the $m$ good targets. We show that when $r$ is not too large,
%i.e., $r = \bigo{\max \left( \frac{m}{KN}, \frac{KN}{m}\right) }$, 
the above lifelong learners can be easily made to work as well as they would when $r=0$. To address greater values of $r$, we first highlight a trade-off between allowing the learner to learn more targets from scratch and learning the remaining targets with more feature evaluations. We then present a technique that strikes the right balance between the two. 
 
Finally, in Section~\ref{sec:lb} we present lower bounds on the performance of a lifelong learner for all values of $r$, including $r=0$ by designing randomized adversaries. Ignoring the sample size $\samplesize$ and other problem-specific parameters, for small $r$ we prove a lower bound of $\bigomega{KN+mK}$ feature evaluations which proves that our above approaches are in fact tight. For sufficiently large $r$, we prove a bound of $\bigomega{mN}$, thereby demarcating a realm of $r$ where lifelong learning is simply futile. 

We present a summary of our results in Tables~\ref{table:upper} and \ref{table:lower} below.%in Appendix~\ref{app:summary}.

\input{tables.tex}

\subsection{Related Work}

Related work in multi-task or transfer learning \citep{halnips12,pm:13,PY10} considers the case where tasks are drawn from an easily learnable distribution or are presented to the learner all at once. The theoretical results in that setting are sample complexity results that guarantee low  error averaged over all tasks \citep{Baxter97,baxter2}. On the other hand, research in lifelong learning has been mostly empirical  \citep{ll-robot,ll-a-star,ll-recom,TP97}. %\todo{references ok?}
There has been a small amount of recent theoretical work \citep{eff-rep,wmv}.
%Past theoretical work in lifelong learning has generally focused on the case that target functions are drawn from some easily learnable distribution \citep{Baxter97} or 
\cite{eff-rep} consider fairly simple targets and commonalities such as linear separators that lie in a common low-dimensional subspace. \cite{wmv} consider a setting where except for a small subset of target functions, each target can be written as a weighted majority vote over the previous ones.  \cite{eff-rep} also consider conjunctions that share a set of conjunctive metafeatures, but assume that the metafeatures contain a unique ``anchor variable''.  Though decision trees have a more elaborate combinatorial structure than conjunctions, in this work we are able to achieve strong guarantees for lifelong learning of decision trees (and other classes) without making such assumptions about the metafeatures. We also note that one of main technical challenges addressed by \cite{eff-rep} is that of controlling error propagation during lifelong learning.  However, for the problems considered in this paper, it is possible to learn targets exactly from scratch, so we do not have to deal with error propagation.

Feature-efficiency has been considered in the single-task setting, often under the name of budgeted learning \citep{LizotteMG03,KapoorG05,BLworkshop10},  where one has to learn an accurate model subject to a limit on feature evaluations, somewhat like bandit algorithms. 
% Our goal is also fundamentally different from prior work such as multi-task feature selection 
\cite{fs,mtl} consider a related problem  in a multi-task setting with all tasks present up-front, where the learner has free access to all  features but uses commonalities between targets to identify useful common features in order to be sample-efficient.

\section{Preliminaries}
\label{sec:model}

In this section, we define our notations (later summarized in Table~\ref{tab:notations}) and present a high level protocol which will provide a framework for presenting our algorithms in the later sections. We consider a setting in which the learner faces a sequence of $m$ related target functions $g^{(j)}$ over the same set of $\N$ 
features/variables (where both $m$ and $N$ are very large).  
The target functions arrive one after the other, each with its own set of training data $\sample^{(j)}$ with at most $\samplesize$ examples to learn from.  Also, feature evaluation (or equivalently, feature query or feature examination) is costly: if we view our training data for $g^{(j)}$ as an $\samplesize \times N$ matrix, we pay a cost of 1 for each cell probed in the matrix.

Our belief is that the targets are related to each other through an unknown set $\F$ of {\em metafeatures} that are parts of functions. More specifically, all targets in the series can be expressed by combining metafeatures in $\F$ using a known set of legal combination rules, such as concatenating lists or trees.  Our algorithms will learn a set of hypothesized metafeatures $\tildeF$ that allows them to learn new targets using a small number of feature evaluations except for a limited number of targets learned from scratch i.e., by examining all features on all examples. We call $\tildeF$ our {\em learned  representation}.  Note that we will refer to $\tildeF$ as just metafeatures if it is clear from context that it does not refer to the true metafeatures $\F$. 

Then, our lifelong protocol is as follows. We make use of two basic subroutines: a \LFD{} routine that uses $\tildeF$ to learn new related targets, and an \ID{} routine that improves our representation $\tildeF$ whenever the first subroutine fails.  We begin with an empty $\tildeF$. On task $j$, using $\tildeF$ and $\sample^{(j)}$, we attempt to cheaply learn target $g^{(j)}$ with \LFD{}. If \LFD{} fails to learn the target, we evaluate all features in $\sample^{(j)}$ and learn $g^{(j)}$ from scratch. Then, we provide $\tildeF$ and $g^{(j)}$ as input to \ID{} to update $\tildeF$.  For clarity, we present this generic approach, which we will call as (\LFD{}, \ID{})-protocol, in Algorithm \ref{alg:generic}. %in Appendix~\ref{app:notations}. 
In the following sections,  we will present concrete approaches for these subroutines, specific to each class of targets.  We will then analyze the performance of the protocol in terms of the total number of feature evaluations (across all samples over all the tasks) given an adversarial stream of tasks.

Our setting can be viewed as analogous to that of dictionary learning  \citep{LS00,EA06,AroraGM14} in which the goal is to find a small set of vectors that can express a given set of vectors via sparse linear combinations.  Here, we will be interested in broader classes of objects and richer types of combination rules.

\input{notations.tex}

%% file: tables.tex
\begin{table}[H]
\begin{center}
\begin{tabular}{|p{8cm}|c|}
\hline 
\textbf{Problem} &  \textbf{Total number of feature evaluations} \\ \hline
Decision trees of depth $d$ & $ \bigo{\samplesize(KN + mKd)}$ \\ \hline
Decision trees of depth $d$ in semi-adversarial model & $ \bigo{\samplesize(\frac{\log K}{p_{\min}}  N + m(K+d))}$ \\ \hline
Decision trees of depth $d$ with anchor variables & $ \bigo{\samplesize(KN + m(K+d))}$ \\ \hline
Decision lists of depth $d$ & $ \bigo{\samplesize(K^2N + m(K^2+d))}$  \\ \hline
Monomials of degree $d$ & $ \logo{KN + m(K+d)}$ \\ \hline
Polynomials of degree $d$, sparsity $t$ & $ \bigo{\samplesize(KN + m(K+td))}$ \\ \hline
\end{tabular} 
\caption{Performance of our approaches}
\label{table:upper}
\end{center}
\end{table}

\begin{table}[H]
\begin{center}
 \begin{tabular}{|c|c|c|c|}
 \hline
 Range of $r$ & Performance of algorithm & Lower bound  \\ \hline 
 $0 \leq r\leq r_{\min}$ & $O(S (NK + Km))$ & $\bigomega{NK + Km}$ \\ \hline 
  $r\in [r_{\min}, r_{\max}]$ & $O(S (\underbrace{\sqrt{rKNm}}_{\leq \sqrt{\frac{r_{\max}}{r}}\max\left( \frac{r}{N-K}, 1 \right)KN } + Km))$ & $\bigomega{\max\left( \frac{r}{N-K}, 1 \right) KN + Km}$ \\ \hline
$r \geq r_{\max}$ & $\bigo{S mN}$ & $\bigomega{mN}$ \\
\hline 
 \end{tabular}
\caption{Performance of our algorithms for different values vs  the lower bounds for different values of $r$. Here, we will define $r_{\min}=\max \left( \frac{ m}{N}, \frac{K N}{m} , K \right)$ and $r_{\max}  = \min \left( \frac{mN}{K} , \frac{(N-K)^2 m}{K N}\right)$}
\label{table:lower}
\end{center}
\end{table}

%% file: notations.tex
\begin{table}[H]
\centering
\begin{tabular}{|c|c|}
\hline
\textbf{Notation} & \textbf{Meaning} \\ \hline
$m$ & No. of targets in sequence \\ \hline
$N$ & No. of features/variables \\ \hline
$\F$ & True metafeature set/representation \\ \hline
$\tildeF$ & Learned representation \\ \hline
$K$ & No. of true metafeatures \\ \hline 
$\samplesize$ & No. of samples for each task\\ \hline
$\sample^{(j)}$ & Training data for task $j$ \\ \hline 
\end{tabular}
\caption{Important notations}
\label{tab:notations}
\end{table}

\begin{algorithm}[H]
\caption{$(\mathcal{A}_{\sf UR}, \mathcal{A}_{\sf IR})$-protocol for lifelong learning}
\label{alg:generic}
\begin{algorithmic}[1]
\STATE \textbf{Input}: A sequence of $m$ training sets $\sample^{(1)}, \sample^{(2)}, \hdots, $
% learning tasks 
corresponding to targets $g^{(1)}, g^{(2)}, \hdots $, each of which can be represented using an unknown set $\F$ of $K$ metafeatures.
\STATE Let $\tildeF$ be our current learned representation. Initialize $\tildeF$ to be empty.
\FOR{$j=1,2, \hdots m$}
	\STATE Using $\tildeF$ and $S^{(j)}$, attempt to cheaply learn $g^{(j)}$ with %Algorithm 
\LFD{} algorithm $\mathcal{A}_{\sf UR}$.
	\STATE If learning was not successful, extract all features in $\sample^{(j)}$ and learn $g^{(j)}$ from scratch; provide $\tildeF$ and  $g^{(j)}$ as input to \ID{}  algorithm $\mathcal{A}_{\sf IR}$ to update $\tildeF$.
\ENDFOR
\end{algorithmic}
\end{algorithm}

%% file: decision-trees.tex
We first formally define decision tree metafeatures and describe our learning model. Based on this we describe our problem concretely in Problem Setup~\ref{ass:dt-general}.  To simplify our discussion, we consider decision trees over Boolean features, though we later present a simple extension to real values. Formally, in a decision tree $g: \{0,1\}^{\N} \to \{+,-\}$, each internal node corresponds to a split over one of $N$ variables and each leaf node corresponds to one of the two labels $\{+,-\}$. No internal node and its ancestor split on the same variable. 

%Recall that for a sequence of $m$ related tasks, we wish to learn decision trees that are related to each other through an unknown set of decision tree substructures or metafeatures $\F$ . More formally, 
Now, we define a metafeature to be an \textit{incomplete} decision tree, a tree where any of the leaf nodes can be empty i.e., the labels of some leaf nodes are left unspecified. Then,  there are two natural ways of combining metafeatures to form a (complete) decision tree. Let $\nodeA$ be one of the empty leaf nodes of a metafeature $f$. We may combine $f$ with another incomplete tree $f'$ using an $\affix(f,\nodeA,f')$ operation which simply affixes the root node of $f'$ at $\nodeA$ (as illustrated in Figure~\ref{fig:affix}). As a result, $\nodeA$ now becomes an internal node of a larger incomplete tree. The variable at $\nodeA$ and its descendants correspond to the variables in $f'$. Alternatively, we may perform a $\lab(f,\nodeA,l)$ operation which assigns a label $l \in \{+,- \}$ to the empty node $\nodeA$ in $f$.
We can then pick an arbitrary element $f \in \F$, apply an arbitrary sequence of $\lab$ and $\affix$ operations (affixing only trees from $\F$) and eventually grow $f$ into a decision tree. In this manner, we define below what it means to be able to represent a decision tree using a set of metafeatures $\F$.  Both $\lab$ and $\affix$ are described for completeness in Appendix~\ref{app:dt}.

\begin{definition}\label{def:dict}
%[\textbf{$\dtspace(\F)$: Decision trees that can be expressed using metafeatures $\F$}]
Let $\F  = \{f_1, f_2, \ldots \}$ where each metafeature $f_i$ is an incomplete decision tree.  We define $\dtspace(\F)$ to be the set of all 
%\textit{valid} 
decision trees that can be grown by using the elements of $\F$ and sequentially applying {\lab} and {\affix} operations on them. 
We say that a decision tree $g$ can be expressed using $\F$ if $g \in \dtspace(\F)$.
\end{definition}
\begin{figure}[H]
\centering
\includegraphics[scale=0.55]{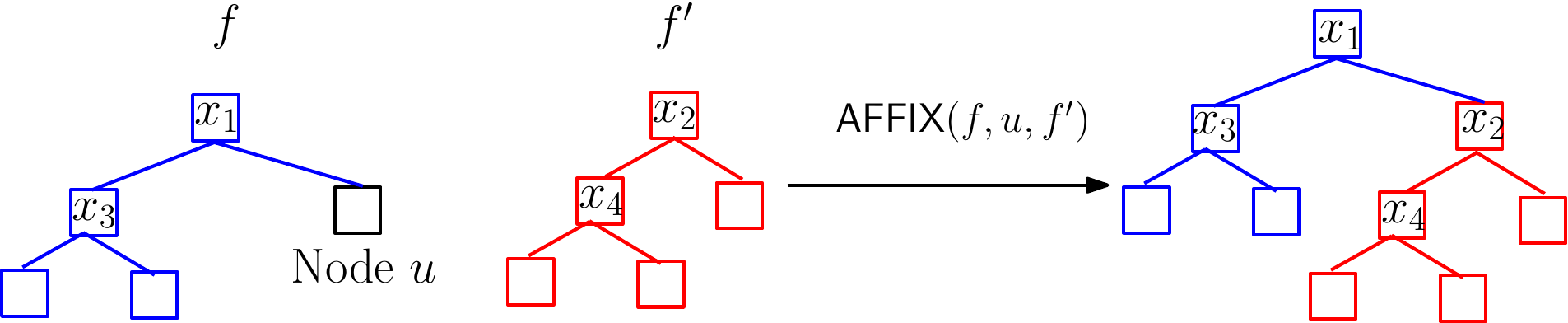}
\caption{Illustration of \affix}
\label{fig:affix}
\end{figure}

A modeling challenge here is that there are no known polynomial-time algorithms to learn decision trees, even ignoring the issue of costly features and even for trees of depth $d = O(\log N)$.    On the other hand, there are popular top-down tree-learning algorithms (like ID3 and C4.5) that work well empirically \citep{RM08,Quinlan86,cart84}.  Therefore, we will assume that we are given such an algorithm that indeed correctly produces $g^{(j)}$ from $\sample^{(j)}$ if we are willing to evaluate all the features in all the examples.  More specifically, these methods are defined by a ``gain function'' $\gain(\sample,i)$ that given a set of labeled examples $\sample$ and a feature $i$, returns a score indicating the desirability of splitting the set $\sample$ using feature $i$.  
%\todo{Input type?} 
%use ``S'' without j here, since its a generic argument to the function
For instance, ID3 uses {\em information gain} as its splitting criterion,\footnote{If feature $i$ splits data set $S$ into two sets $L$ and $R$, its information gain of feature $i$ is then $Ent(S) - [\frac{|L|}{|S|}Ent(L) + \frac{|R|}{|S|}Ent(R)]$.  Here, $Ent$ is the binary entropy of the label proportions in the given set; that is, if a $p$ fraction of the labels in $S'$ are positive, then $Ent(S')=p\log_2(1/p)+(1-p)\log_2(1/(1-p))$.} and an elegant theoretical analysis of the use of different such gain functions is given in \cite{KM}.   The algorithm begins at the root, chooses the variable of highest gain to put there, and then recurses on the nodes on each side.  This process continues until all leaves are pure (all positive or all negative). 
%\footnote{In practice, the algorithms may halt when leaves are almost pure.}  
%Our belief is that we are given a function $\gain$ such that if all features were evaluated on all examples, this process indeed would produce $g^{(j)}$ from  $\sample^{(j)}$ for all $j$.  
%Finally, to verify whether the algorithm succeeds at a task, we assume that the data sets $\sample^{(j)}$ are large enough and we are given $s$ and $d$ such that if the output is a decision tree of size at most $s$ and depth at most $d$ that is consistent with $\sample^{(j)}$, we consider the learner to be successful. %
\begin{assumption}\label{ass:dt-general} The 
decision tree targets $g^{(1)}, \hdots g^{(m)}$ and data sets $\sample^{(1)}, \hdots, \sample^{(m)}$, each of at most $\samplesize$ examples, satisfy the following conditions:

\begin{enumerate}
%\item No two internal nodes along any path down the root of $g^{(j)}$ share the same variable. In other words, we say that $g^{(j)}$ is a {\em valid} decision tree.
\item There exists an unknown set
$\F$ of $K$ metafeatures ($K \ll N$) such that $\forall j$, $g^{(j)} \in \dtspace(\F)$.
\item The target $g^{(j)}$ can be learned by running top-down decision-tree learning on $\sample^{(j)}$ using a given $\gain$ function. In other words, 
%given 
%access to all features of the examples in,
{\em always choosing to recursively split on the variable of highest $\gain$ based on $\sample^{(j)}$ produces $g^{(j)}$}.
\item We are given $s, d$ ($d \ll N$) such that $g^{(j)}$ has  at most $s$ internal nodes and depth at most $d$. Then, $\samplesize=\bigo{s \log N}$ examples are sufficient to guarantee that $g^{(j)}$ has high accuracy over the underlying distribution over data. 
\end{enumerate}
\end{assumption}
A straightforward lifelong learning approach would be as follows: \ID{} simply adds to $\tildeF$ features seen in tasks learned from scratch as metafeatures, and \LFD{} examines only those (meta)features in $\tildeF$ when learning a target. Since each metafeature in $\F$ can have at most $s$ distinct features, this learns at most $K$ targets from scratch and evaluates only $Ks$ features per example on the rest i.e., $\bigo{ \samplesize (K N + m Ks) }$ feature queries overall (see Appendix~\ref{app:dt} for details). However, when $s=\bigomega{N}$ this is no better than learning all tasks individually from scratch. In this section, we will present a significantly better protocol:

\begin{restatable}{theorem}{dt}
\label{thm:dt}
%In the model of Problem Setup~\ref{ass:dt-general}  for decision trees, 
The  (\LFD{} Algorithm~\ref{alg:lfd-dt}, \ID{} Algorithm~\ref{alg:id-dt})-protocol for decision trees makes $\bigo{\samplesize(K N + m  Kd)}$ feature evaluations overall and runs in time $poly(m,N,K,\samplesize, s,d)$.\footnote{
It may seem that this result can be equivalently stated in terms of the average number of features examined per example i.e., $\bigo{K N + m  Kd}$. However, such a performance metric is different from what we defined. Under certain independence conditions it may be possible to learn a target simply by drawing a large number of examples and examining only a single feature per example while still making many feature evaluations in total.} % is $\left(\bigo{Kd}, K\right)$-\efficient{}. 
 \end{restatable}
%A straightforward solution (described for completeness in Appendix~\ref{app:dt})  simply keeps a record of the features  that have been seen in previous tasks, makes  This is however large considering that $s = \bigo{2^d}$. 
This is a significant improvement especially in the case of shallow bushy trees for which $d\ll s $ e.g., when $d = \bigo{\log N}$ but $s = \bigomega{N}$. To achieve this improvement, we need a computationally efficient approach that extracts bigger decision tree substructures from previous tasks and also knows how to learn future tasks using such a representation. We first address the latter problem: we present an \LFD{} routine, Algorithm~\ref{alg:lfd-dt}, that takes as input a set of hypothesized metafeatures $\tildeF$ and a training dataset $\sample$ consistent with an unknown tree $g$ and either outputs a consistent tree $\tilde{g}$ or halts with failure. To appreciate its guarantees,
%we need to set up some notations. In particular, we will 
define $\Pref(f)$ to denote the set of all ``prefix'' trees (prunings) of some incomplete tree $f$.  For any set of hypothesized metafeatures $\tildeF$, let $\Pref(\tildeF) =\{\Pref(\tilde{f}) \; | \tilde{f}  \in \tildeF\}$. We show that Algorithm~\ref{alg:lfd-dt}, given $\tildeF$, can effectively learn a target that can be represented  using not only $\tildeF$, but also the exponentially larger metafeature set $\Pref(\tildeF)$. That is, our \LFD{} algorithm can effectively learn trees from a much larger space $\dtspace(\Pref(\tildeF))$ compared to just $\dtspace(\tildeF)$. \\

We now describe Algorithm~\ref{alg:lfd-dt}. Though we limit our discussion to Boolean feature values for simplicity, we later extend it to real values.  In Algorithm~\ref{alg:lfd-dt}, we basically grow an incomplete decision tree $\tilde{g}$ one node at a time, by picking one of its empty leaf nodes $\nodeA$, and either assigning a label to $\nodeA$ or splitting $\nodeA$ on a particular feature. Before doing so, we first make sure that we have not failed already (Step~\ref{step:failure}). More specifically, if $\nodeA$ is at a depth greater than $d$ or if $\tilde{g}$ already has more than $s$ nodes, we halt with failure because we were not able to find a small tree consistent with the data.  If not, we proceed to examine samples from the training set that have reached $\nodeA$, which we will denote by $\sample_{\nodeA}$. If all $x \in \sample_{\nodeA}$ have the same label, we make $\nodeA$ a leaf with that label and proceed to other nodes in $\tilde{g}$. 
 
Otherwise, we evaluate a small set of features on $\sample_{\nodeA}$ to compute their $\gain$ and pick the best of those features to be the variable at $\nodeA$ (denoted by $\Var(\nodeA)$). The way we pick this set of features at $\nodeA$, which we will call $\I$, is based on the following intuition. Assume we have grown $\tilde{g}$ identically to $g$ so far and let $\nodeA'$ be the node in $g$ that corresponds to $\nodeA$. Then the correct variable to be assigned at $\nodeA$ is $\Var(\nodeA')$ which is in fact the gain maximizing variable on $\sample_{\nodeA}$ (as assumed in the second point of Problem Setup~\ref{ass:dt-general}).  Thus, our goal is to ensure $\Var(\nodeA')\in \I$.  

If indeed $g \in \dtspace(\Pref(\tildeF))$, this variable must in fact correspond to the variable in some node in some $\tilde{f} \in \tildeF$. In other words, we should be able to ``superimpose'' some $\tilde{f}$ over $\tilde{g}$ with the root of $\tilde{f}$ at either $\nodeA$ or one of its ancestors such that the variable in $\tilde{f}$ that has been superimposed over $\nodeA$ is in fact the correct variable for $\nodeA$.  Additionally,
%if $\tilde{f}$ is to be superimposed with its root at one of $\nodeA$'s ancestors, 
the variables in $\tilde{f}$ should not conflict with those that have already been assigned to the ancestors of $\nodeA$ in $\tilde{g}$.  Since we do not know which $\tilde{f}$ and which superimposition of $\tilde{f}$ induces the correct variable at $\nodeA$, we add to $\I$ the variable induced at $\nodeA$ by every possible superimposition: we pick every $\tilde{f} \in \tildeF$ and every node $\nodeC$ that is either an ancestor of $\nodeA$ or $\nodeA$ itself, and then superimpose $\tilde{f}$ over $\tilde{g}$ with its root at $\nodeC$. We add to $\I$ the variable thus induced at $\nodeA$, provided the variables in $\tilde{f}$ do not conflict with those in the ancestors of $\nodeA$. In Algorithm~\ref{alg:lfd-dt}, we use helper routines, $\induce(\tilde{g}, \nodeC, \nodeA, \tilde{f})$ which outputs the induced variable and $\conflict(\tilde{g}, \nodeC, \nodeA, \tilde{f})$ which outputs false if there is no conflict (both these simple subroutines are described for completeness in Appendix~\ref{app:dt} and illustrated in Figure~\ref{fig:conflict-induce}). Finally, since no variable should repeat along any path down the root, we remove from $\I$ any variable already assigned to an ancestor of $\nodeA$. Then, we assign the gain maximizing feature from $\I$ to $\nodeA$. 

Observe that, at $\nodeA$, in total over all $\tilde{f}$ we may examine $O(|\tildeF|d)$ features on $\sample_{\nodeA}$. Therefore, for a particular sample, considering all nodes along a path from the root, we may examine $O(|\tildeF|d^2)$ features. However, with a more rigorous analysis we prove a tighter bound:

\begin{restatable}{theorem}{lfddt}
\label{thm:lfd-dt}
 \LFD{} Algorithm~\ref{alg:lfd-dt} has the property that given $\tildeF$ and data $\sample$,  a) if the underlying target $g \in \dtspace(\Pref(\tildeF))$, the algorithm outputs $g$ and b) conversely, if the algorithm outputs $\tilde{g}$ without halting on failure, then $\tilde{g}$ has depth at most $d$,  size at most $s$ and is consistent with $\sample$, %(according to Definition~\ref{def:dt-success}).  
c) the algorithm evaluates $O(|\tildeF|+d)$ features per example. %	. \footnote{Observe that this is independent of the sparsity i.e., the number of dictionary elements used along any path in the tree.}
\end{restatable}

\begin{algorithm}[H]
\caption{\LFD{} - Learning a decision tree using metafeatures \label{alg:lfd-dt}}
\begin{algorithmic}[1]
\STATE \textbf{Input:} Metafeatures $\tildeF$, samples $\sample$ consistent with unknown $g$, depth bound $d$, size bound $s$.
%\STATE Draw $\bigo{tk^t \log k}$ samples, $S$
%\STATE Let $\I = \{\tilde{a}_1, \tilde{a}_2, \hdots \tilde{a}_k\}$ be the set of features that occur at the root of the trees in $\tildeF$. Query for the value of each feature in $\I$ on each example in $\sample$.
\STATE Initialize the tree $\tilde{g}$ to be an empty leaf node. Let $\Z$ be the set of empty leaf nodes in $\tilde{g}$.
\WHILE{$\exists \; \nodeA \in \Z$} \label{step:choose}
	\STATE \label{step:failure} Halt with failure if a) $\nodeA$ is at depth $ > d$ or b) the size of $\tilde{g}$ is $ > s$.
	\STATE Let $\sample_{\nodeA}$ be the examples that have reached $\nodeA$.
	\IF{all $x \in \sample_\nodeA$ have the same label $l$} \label{step:leaf}
		\STATE Make $\nodeA$ a leaf with the label $l$.
	\ELSE 
		\STATE Let $\I$ be the set of features to be examined at $\nodeA$. Initialize $\I$ to be empty.
		\FOR{ each $\tilde{f} \in \tildeF$ and each node $\nodeC$ in the path starting from the root of $\tilde{g}$ to $\nodeA$}
		\STATE \label{step:superimpose} If {$\conflict(\tilde{g}, \nodeC, \nodeA, \tilde{f})$ is false}, add $\induce(\tilde{g}, \nodeC, \nodeA, \tilde{f})$ to $\I$.
		\ENDFOR
		\STATE \label{step:valid} Remove from $\I$ any variable assigned to an ancestor of $\nodeA$.
		\STATE \label{step:assign}Evaluate only the features $\I$ on $\sample_\nodeA$. Assign $\Var(\nodeA) \gets \arg\max_{i \in \I}\gain(\sample_\nodeA, i)$.  %Add both the children nodes of $\nodeA$ to $\Z$ while removing $\nodeA$ from it. 

	\ENDIF
%		\STATE Let $i_{\tilde{f}}$ be the root variable of $\tilde{f}$.
%		\IF{$\exists$ an ancestor $\nodeB$ of $\nodeA$ such that $\Var(\nodeB) = i_{\tilde{f}}$} \label{step:feature}
%			\STATE \label{step:invalid-dtree-1}  If $\conflict(\tilde{g}, \nodeB, \tilde{f})$ is false, then add $\induce(\tilde{g}, \nodeB, \tilde{f}, \nodeA)$ to $\I$. 
%		\ELSE
%			\STATE \label{step:invalid-dtree-2} Add $\induce(\tilde{g},\nodeA,\tilde{f},\nodeA)$, in other words $i_{\tilde{f}}$, to $\I$.
%		\ENDIF
\ENDWHILE
\STATE Output $\tilde{g}$.
\end{algorithmic}
\end{algorithm}

\begin{proof}
(a) and  (c) follow from Lemma~\ref{lem:lfd-correctness} and Lemma~\ref{lem:lfd-dt-feature-efficiency} respectively, which we prove below. (b) follows immediately from the algorithm, more specifically from Step~\ref{step:failure} and ~\ref{step:leaf}. We need this guarantee so that when the learner does not fail, its output is guaranteed to be correct.
\end{proof}

\begin{lemma}
\label{lem:lfd-correctness}
If  $g \in \dtspace(\Pref(\tildeF))$ %and $g$ is a valid decision tree
, Algorithm~\ref{alg:lfd-dt} outputs $\tilde{g} = g$. 
\end{lemma} 

\begin{proof} 
We are given that $g \in \dtspace(\Pref(\tildeF))$. We will show by induction that $\tilde{g}$ is always grown correctly i.e., $\tilde{g} \in \Pref(g)$. This is trivially true at the beginning.  Consider the general case. Let $\nodeA$ be the node in $\tilde{g}$ that is chosen in Step~\ref{step:choose} to be grown. By our induction hypothesis that $\tilde{g}$ is a prefix of $g$, there exists $\nodeA'$ in $g$ that corresponds to $\nodeA$ and furthermore, $\sample_{\nodeA} = \sample_{\nodeA'}$. Now to show that $\nodeA$ will be grown identical to  $\nodeA'$, since $\tilde{g}$ is only a prefix, the size and depth constraints will be satisfied and so we are guaranteed to not halt with failure at this node. Next, if $\nodeA'$ was a leaf node, since $\sample_{\nodeA} = \sample_{\nodeA'}$, we are guaranteed to label $\nodeA$ as a leaf and assign it the correct label. 

If $\nodeA'$ is not a leaf node, let $\Var(\nodeA')$ be the variable present in $\nodeA'$ i.e.,   $\Var(\nodeA')=\arg\max_{i \in [N]}\gain(\sample_{\nodeA'}, i)$. Therefore, to show that we assign  $\Var(\nodeA')$ to $\nodeA$ in Step~\ref{step:assign}, we only need to prove that $\Var(\nodeA') \in \I$ i.e., we consider this feature for examination.  To prove this, note that in $g$, $\Var(\nodeA')$ belongs to the prefix of some metafeature $\tilde{f}^*$ from $\tildeF$ that is rooted either at  some $\nodeB'$ which is either $\nodeA'$ itself or at one of its ancestors (because $g \in \dtspace(\Pref(\tildeF))$). We can show that in Step~\ref{step:superimpose}, when $\nodeC = \nodeB$ and $\tilde{f} = \tilde{f}^*$, we end up adding $\Var(\nodeA')$ to $\I$. First, if $\nodeB$ is the corresponding node in $\tilde{g}$ we will have that $\conflict(\tilde{g}, \nodeB, \nodeA, \tilde{f}^*)$ is false. Furthermore, clearly $\induce(\tilde{g}, \nodeB, \nodeA, \tilde{f}) = \Var(\nodeA')$.  Now since $g$ has no variable repeating along any root-to-leaf path, $\Var(\nodeA')$ does not occur in any of the ancestor nodes of $\nodeA'$, and similarly in $\tilde{g}$, it does not occur in any of the ancestor nodes of $\nodeA$. Thus, the conditions in Step~\ref{step:superimpose} succeed, following which $\Var(\nodeA')$ is added to $\I$. 
\end{proof}

\begin{lemma}
\label{lem:lfd-dt-feature-efficiency}
Algorithm \ref{alg:lfd-dt} makes at most $O(|\tildeF| + d)$ feature queries per example.
\end{lemma}

\begin{proof} 
First of all note that each example corresponds to a particular path in $\tilde{g}$. Thus, the features examined on that example as $\tilde{g}$ was grown, correspond to the different features computed from $\induce(\tilde{g}, \nodeC, \nodeA, \tilde{f})$ for different nodes $\nodeB$ and $\nodeA$ on that path.
These feature queries can be classified into two types depending on whether A) $\nodeC = \nodeA$ or  B) $\nodeC$ is an ancestor of $\nodeA$. For type A, since $\nodeC = \nodeA$, $\induce(\tilde{g}, \nodeC, \nodeA, \tilde{f})$ can only be one of the fixed set of features that occur at the root of metafeatures in $\tildeF$. In total this may account for at most $|\tildeF|$ feature examinations. 

Now consider the type B features queries corresponding to $\nodeC \neq \nodeA$. Each feature examined in this case corresponds to a 3-tuple $(\nodeC, \nodeA, \tilde{f})$ where $\nodeC$ is an ancestor of $\nodeA$. We claim that for a given $\tilde{f}$, $\nodeC$ has to be unique in this path. 
This is because
$\Var(\nodeC)$ must equal the root variable of $\tilde{f}$ by definition of
$\induce$, and any given variable appears at most once on any path by Step~\ref{step:valid}.

Thus type B feature query effectively corresponds to a 2-tuple $(\nodeA, \tilde{f})$ instead of a 3-tuple $(\nodeC, \nodeA, \tilde{f})$ because $\tilde{f}$ corresponds to a unique $\nodeC$. Let $\nodeC_{\tilde{f}}$ denote this unique node for $\tilde{f}$.  Now, let $k_{\nodeA}$ be the number of type B feature queries made at $\nodeA$. We can divide this case further into type B(a) consisting of nodes $\nodeA$, such that $k_{\nodeA} = 1$ and type B(b) corresponding to $k_{\nodeA} > 1$. In total over the $d$ nodes in $\tilde{g}$, we would examine only $d$ type B(a) features.  Now, for type B(b), at node $\nodeA$, where we evaluate  $k_{\nodeA} $ features at $\nodeA$, we claim that this eliminates at least $k_{\nodeA}-1 $ different metafeatures from resulting in feature examinations of type B further down this path. This is because each of the $k_{\nodeA}$ features that we examine at $\nodeA$ correspond to  $\induce(\tilde{g}, \nodeC_{\tilde{f}}, \nodeA, \tilde{f})$ for some $\tilde{f} \in \tildeF$. Let this set of metafeatures be $\tildeF_{\nodeA}$, where $|\tildeF_{\nodeA}| = k_{\nodeA}$. Now, we assign only one feature to $\nodeA$ that corresponds to say, $\tilde{f}^*\in \tildeF_{\nodeA}$. After this, when we are growing a descendant node $\nodeB$, for the $k_{\nodeA}-1$ other metafeatures $\tilde{f} \in \tildeF_{\nodeA}$ and $\tilde{f} \neq \tilde{f}^*$, $\conflict(\tilde{g}, \nodeC_{\tilde{f}}, \nodeB, \tilde{f})$ will be true as there will be a conflict at $\nodeA$. 
However, since $\conflict(\tilde{g},  \nodeC_{\tilde{f}}, \nodeB,\tilde{f})$ needs to be false in Step~\ref{step:superimpose} for $\tilde{f}$ to result in a feature query, we conclude that  there are $k_{\nodeA}-1$ different metafeatures that do not result in a feature query beyond this point. 

Using the above claim, we can now bound $\sum_{\nodeA: k_{\nodeA} > 1} k_{\nodeA}$, which will account for the total feature queries of type B(b) along the path. Since  $k_{\nodeA}-1$ denotes the number of eliminated metafeatures beyond $\nodeA$, and since only at most $|\tildeF|$ can be eliminated, we have $\sum_{\nodeA: k_{\nodeA} > 1} (k_{\nodeA} - 1) \leq |\tildeF|$. Now, since $\sum_{\nodeA: k_{\nodeA} > 1}1 \leq d$, we have that $\sum_{\nodeA: k_{\nodeA} > 1} k_{\nodeA} \leq |\tildeF| + d$ i.e., we make at most $|\tildeF| +d$ type B(b) feature queries of the last kind on this path. 
Thus, in summary, we examine at most $O(|\tildeF|+d)$ features on each example.
\end{proof}

% As a warm-up, observe how Theorem~\ref{thm:lfd-dt} immediately allows us to analyze a semi-adversarial scenario  where each element of $\F$ has a reasonable chance  (probability at least $p_{\min}$) of being the topmost metafeature in any target. We just learn the first $\logo{1/p_{\min}}$ targets from scratch and add them to $\tildeF$ so that with high probability, each metafeature from $\F$ is guaranteed to be the prefix of some element in $\tildeF$. Then we can use Algorithm~\ref{alg:lfd-dt} to learn the remaining targets using $\tildeF$ as all those targets will lie in $\dtspace(\tildeF)$. By applying Lemma~\ref{lem:lfd-dt-feature-efficiency}, it follows that we will learn those targets by examining  $\bigo{\frac{1}{p_{\min}}\log K+d}$ features per example. In fact, with a more careful analysis we prove a better result in Appendix~\ref{app:dt} (Theorem~\ref{thm:semi-random-dt}). 
 %To proceed to the original adversarial model from the semi-adversarial one, note that some elements of $\F$ may never occur at the topmost position of \textit{any} target. So a technique that simply places trees learned from scratch into $\tildeF$ will not work. 
  %Note that every root-leaf path in $g$ corresponds to a path in $\tilde{g}$ consisting of its root nodes and descendant nodes chosen in the same way.
  %Before we present our \ID{} routine, we state an efficiency result in a semi-random model. 
Now, to provide a lifelong learning protocol for Problem Setup~\ref{ass:dt-general}, the challenge is to design a computationally efficient \ID{} routine\footnote{As a warm-up, consider a semi-adversarial scenario where each element of $\F$ has a reasonable chance of being the topmost metafeature in any target. We can then learn the first few targets from scratch and simply add them to $\tildeF$ so that with high probability, each metafeature from $\F$ is guaranteed to be the prefix of some element in $\tildeF$. Then we can use Algorithm~\ref{alg:lfd-dt} to learn the remaining targets using $\tildeF$ as all those targets will lie in $\dtspace(\Pref(\tildeF))$. We provide a careful analysis of this simpler case in Appendix~\ref{app:dt} Theorem~\ref{thm:semi-random-dt}.}.
 To this end, we present Algorithm~\ref{alg:id-dt} that creates useful metafeatures 
%that are not just previously-seen target functions, 
by adding to $\tildeF$ well-chosen subtrees from target functions. In particular, after learning a target $g$ from scratch, we identify a root-to-leaf path in $g$ that we failed to learn using $\tildeF$. We add to $\tildeF$ the subtrees rooted at every node in that path.  The intuition is that one of these subtrees  makes the representation more useful.  To describe how the path is chosen, let $\tilde{g}$ be the incomplete tree learned using $\tildeF$ just before we halted with failure. Since either the depth or the node count was exceeded in $\tilde{g}$, there must be a path from the root of $\tilde{g}$ longer than the corresponding path in $g$. We pick the corresponding path in $g$ which was incorrectly learned in $\tilde{g}$ (see Figure~\ref{fig:path}).

Finally, as we see below in the proof sketch for Theorem~\ref{thm:dt}, the resulting protocol evaluates only $\bigo{Kd}$ features per example when learning from $\tildeF$, besides learning $K$ trees from scratch. Recall that this is a significant improvement of our straightforward \LFD{}  which evaluates $\bigo{Ks}$ features per example. 
In Appendix~\ref{app:dt}, we present results for more models for decision trees. % with better guarantees.
 %Though this is not as good as our result in the easier semi-adversarial model where $d$ is only an additive factor, this is still an exponential improvement over the baseline in the adversarial model where the feature evaluations scale with $Ks$.

%To understand how a useful path from $g$ is chosen when we fail at learning $g$ using $\tildeF$,We will pick a path in $g$ that was not successfully learned in that its counterpart in $\tilde{g}$ does not have the same number of internal nodes. In particular, if we failed because we exceeded the depth $d$ on a particular path in $\tilde{g}$, we pick the corresponding path in $g$. If we failed because we exceeded the node count $s$, then there must exist a path in $\tilde{g}$ that is deeper than its counterpart in $g$ and we pick this counterpart from $g$. 

\begin{algorithm}[H]
\caption{\ID{} - Decision Trees}
\label{alg:id-dt}
\begin{algorithmic}[1]
\STATE Input: Old representation $\Fold$ and a tree $\g \in \dtspace(\F)$ learned from scratch and the (incorrect) incomplete tree $\tilde{\g}$ learned using $\Fold$.
\STATE $\tildeF \gets \Fold$
\STATE Identify a path from root of $\tilde{g}$ such that the corresponding path in $g$ has fewer internal nodes.
\STATE For each node in the corresponding path in $g$, add the subtree rooted at that node to $\tildeF$.
\STATE Output $\tildeF$
\end{algorithmic}
\end{algorithm}

\begin{proof}(for Theorem~\ref{thm:dt})
We will show by induction that at any point during a run of the protocol, if $k$ targets have been learned from scratch, then there exists a subset of $k$ true metafeatures $\F' \subseteq \F$ that have been ``learned'' in the sense that $f \in \F'$  is the prefix of some metafeature in $\tildeF$, implying that $\dtspace(\F') \subseteq \dtspace(\Pref(\tildeF))$. Then after learning $K$ targets from scratch, it has to be the case that $\F' = \F$ after which $\dtspace(\F) \subseteq \dtspace(\Pref(\tildeF))$ and hence from Lemma~\ref{lem:lfd-correctness} it follows that the protocol can never fail while learning from $\tildeF$.

 The base case is when $\tildeF'$ is empty for which the induction hypothesis is trivially true. Now, assume at some point we have metafeatures $\Fold$ and these correspond to true metafeatures $\F'_{\text{old}}\subseteq \F$ such that  $\dtspace(\F'_{\text{old}}) \subseteq \dtspace(\Pref(\tildeF))$ and $|\F'_{\text{old}}| = k$. Now, from Theorem~\ref{thm:lfd-dt}, we can conclude that any target that lies in $\dtspace(\F'_{\text{old}})$ will be successfully learned by \LFD{} Algorithm~\ref{alg:lfd-dt}. Hence, when \LFD{} does fail on a new target $\g$, it means that the $\g$ contains metafeatures from $\F- \F'_{\text{old}}$. In fact, along any path in $\g$ in which learning failed (that is, the tree $\tilde{g}$ that is output differs from $g$ on this path), there must be a node at which some metafeature from $\F - \F'_{\text{old}}$ is rooted. If this was not true for a particular failed path, we can show using an argument similar to Lemma~\ref{lem:lfd-correctness} that this path would have been learned correctly. Therefore, when we add to $\tildeF$ all the subtrees rooted at the nodes in some failed path in $\g$, we are sure to add a tree which has some $f \in \F - \F'_{\text{old}}$ as one of its prefixes. This means that for the updated set of metafeatures, there exists $\F' = \F'_{\text{old}} \cup \{f\}$ of cardinality $k+1$ that satisfies the induction hypothesis.
 
Now, each time \LFD{} fails, we add at most $d$ metafeatures to $\tildeF$, so $|\tildeF| \leq Kd$. From Theorem~\ref{thm:lfd-dt}, it follows that we evaluate only $\bigo{Kd}$ features when learning using $\tildeF$. %a feature-efficiency of $\bigo{Kd}$ follows.
\end{proof}

\begin{figure}[H]
    \centering
    \begin{minipage}[b]{0.4\textwidth}
        \centering
\includegraphics[scale=0.65]{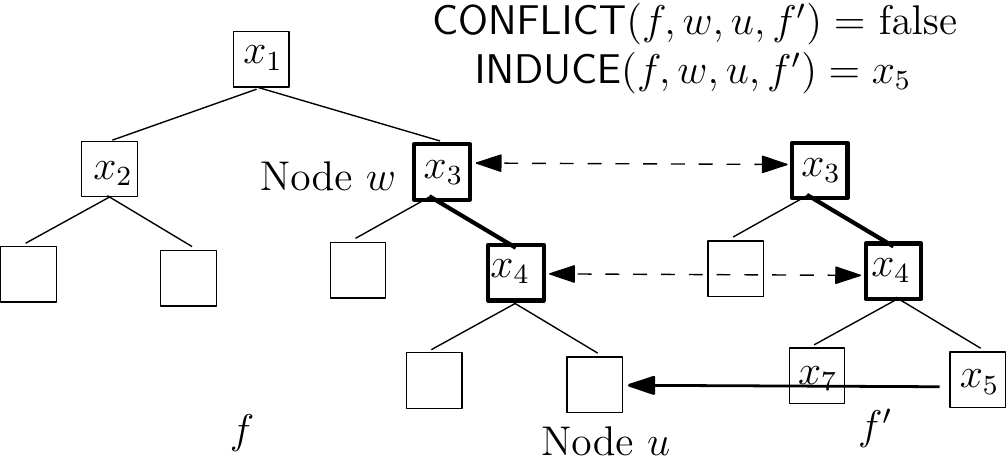}
\caption{Superimposing $f'$ over $f$ with its root at $\nodeC$} %{\conflict{}  and \induce}: 
%$\conflict(f,\nodeC, \nodeA, f')$ is false  and $\induce(f,\nodeC, \nodeA, f')=x_5$}
% The dotted arrows illustrates how $f'$ is superimposed over $f$ with its root at node $\nodeC$. 
%In $f$ the bold path denotes the set of nodes (all ancestors of $\nodeA$ that are not ancestors of $\nodeC$ themselves) that are checked for conflicts. Since the corresponding nodes in $f'$ have the same variables, $\conflict(f,\nodeC, \nodeA, f')$ outputs false. Furthermore, $\induce(f,\nodeC, \nodeA, f')$ outputs $x_5$.}
\label{fig:conflict-induce}
    \end{minipage}%
    \qquad
    \begin{minipage}[b]{0.5\textwidth}
       \centering
		\includegraphics[scale=0.65]{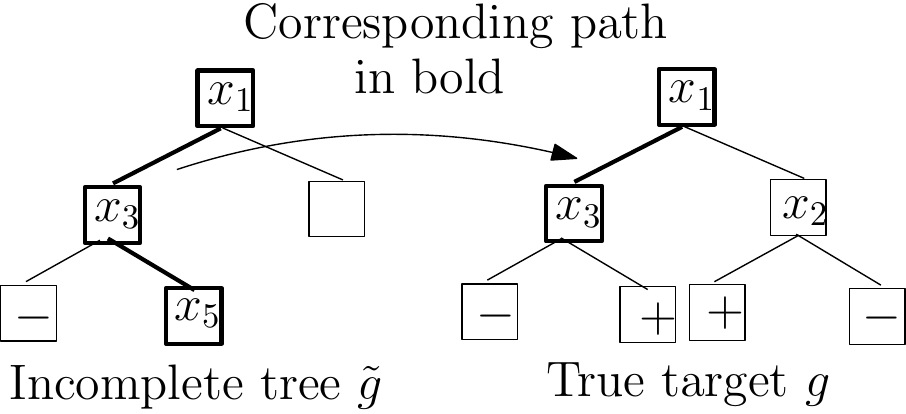}
		\caption{Path chosen by \ID{} Algorithm~\ref{alg:id-dt}.}
		%The chosen path is marked in bold. On the right is the incomplete tree learned by \LFD{} Algorithm~\ref{alg:lfd-dt}. In bold is a path that has more variables than it should, as seen on the left in the correct tree $g$.
		\label{fig:path}
    \end{minipage}
\end{figure}

   \noindent
 \textbf{Extension to real-valued features}: Our results hold also for decision trees over real-valued features, where nodes contain binary splits such as ``$x_1 \geq 7$''.  In particular, we reduce this to the Boolean case by viewing each such split as a Boolean variable.  While this reduction involves an implicitly infinite number of Boolean variables, our bounds still apply. This is because we make only $N$ feature evaluations per example when learning from scratch (and not infinitely many). Also, the feature evaluations made by our \LFD{} is independent of the number of Boolean variables.

%but since our bounds do not depend on $s$ or $N$ at all, they still apply.
\subsection{Decision Lists} %\label{sec:dl}
While we can use the above protocol for decision lists too, it does not effectively provide any improvement over the baseline approach because for lists, $s=d$. However, by making use of the structure of decision lists, we provide a protocol, namely \ID{} Algorithm~\ref{alg:id-dl}, that learns $K^2$ lists from scratch and on the rest examines only $\bigo{K^2+d}$ features per example. The high level idea is that when we fail to learn a target using $\tildeF$, we add to $\tildeF$ only a single suffix of the target list as a new metafeature instead of adding all $d$ suffixes like in Algorithm~\ref{alg:id-dt}.

More specifically,  given a decision list $g$ learned from scratch (that we could not learn from $\tildeF$), we examine $g$ and the actual list $\tilde{g}$ we learned from $\tildeF$. Then we simply ignore the first few nodes of $g$ that we managed to learn using $\tildeF$, and add the remaining part of the list to $\tildeF$. The intuition is that the representation is improved by introducing a part of $g$ that we could not learn using $\tildeF$. 
Note that here $\tilde{g}$ might not even be a complete decision list as \LFD{} may have simply failed in finding a decision list using $\tildeF$ that fits the data. However, it may have still been successful in learning the first few nodes of $g$. 

Here, we use the term suffix to denote a subtree (i.e., sublist) of a list. In other words, a suffix of a list would be a path beginning anywhere on the list and ending at  the leaf. Similarly, we use the term prefix to denote a path beginning at the root of the list and ending anywhere on the list. In our proof, we will  use the notation $\nodeA \dlin f$ to denote that node $\nodeA$ is present in the list $f$, and $f' \dlsubset f$ to denote that $f'$ is an incomplete list (like an incomplete decision tree) which corresponds to a path within $f$, not necessarily a prefix or a suffix. 
Furthermore, if $g$ is a concatenation of other (incomplete) lists $g_1, g_2, \hdots$ we will say $g = (g_1, g_2, \hdots)$.

\begin{algorithm}[H]
\caption{\ID{} - Decision Lists}%/Boosting }
\label{alg:id-dl}
\begin{algorithmic}[1]
\STATE \textbf{Input}: Old representation $\Fold$, target $g$ learned from scratch, $\tilde{g}$ learned using $\Fold$.
\STATE Let $g = (g_p, g_s)$ where $g_p$ is the longest common prefix of $\tilde{g}$ and $g$. 
\STATE $\tildeF \gets \Fold \cup \{g_s\}$
\STATE Return $\tildeF$
\end{algorithmic}
\end{algorithm}

We now present an outline of our proof for the claim that employing \LFD{} Algorithm~\ref{alg:lfd-dt} along with \ID{} Algorithm~\ref{alg:id-dl} learns at most $\bigo{K^2}$ decision lists from scratch.  A crucial fact we use is that \LFD{} Algorithm~\ref{alg:lfd-dt} learns any list iff it belongs to $\dtspace(\Pref(\tildeF))$.  Now, observe that there must exist an $f \in \F$ such that $f$ is a part of $g$ and furthermore, \LFD{}  was able to learn upto a prefix $f_p$ of $f$ after which it failed to learn the remaining suffix of $f$, say $f_s$. Our result follows if we can show that 
there can only be $O(K)$ failures of \LFD{} that correspond to a particular $f$ in this manner. To prove this, we will categorize the failures of  \LFD{} corresponding to $f$ based on whether $f_p \in   \dtspace(\Pref(\Fold))$ and show that there can be only $\bigo{K}$ failures for each case, for a given $f$. 

When $f_p \in   \dtspace(\Pref(\Fold))$, after running \LFD{} Algorithm~\ref{alg:lfd-dt}, we will have that $f_s \in \dtspace(\Pref(\tildeF))$ because $g_s$ which has the prefix $f_s$ was added to our representation. Then, $f \in \dtspace(\Pref(\Fold))$, and therefore on any new target there can not be a failure corresponding to $f$. Thus, there is at most one failure corresponding to $f$, of this type.

The case where $f_p \notin   \dtspace(\Pref(\Fold))$ requires a more intricate argument which is based on identifying another $f'$ chosen carefully from an ``indirect'' representation of $g$ in terms of $\F$. In particular, on one hand there is a direct representation of $g$ in terms of $\F$. At the same time, since Algorithm~\ref{alg:lfd-dt} learned $g_p$ using $\Fold$, $g_p$ can be represented as a sequence of prefixes from $\Fold$. Since each element in $\Fold$ is also from $\dtspace(\F)$, we can indirectly represent this sequence of prefixes in terms of parts of metafeatures from $\F$. We will choose an appropriately positioned $f'$ from this representation and show that there can be only two failures corresponding to a particular $f$ and $f'$. Thus, there can only be $\bigo{K}$ failures for a particular $f$.\\

\begin{theorem}
\label{thm:dl}
%In the model of Problem Setup~\ref{ass:dt-general}  for decision lists, 
The  (\LFD{} Algorithm~\ref{alg:lfd-dt}, \ID{} Algorithm~\ref{alg:id-dl})-protocol for decision lists
makes $\bigo{\samplesize(K^2 N + m (K^2 + d))}$ feature evaluations overall and runs in time $poly(m, N, K, \samplesize, d)$.
%is $(\bigo{K^2+d}, K^2)$-\efficient{}. 
\end{theorem}

\begin{proof}
We show that the protocol learns at most $\bigo{K^2}$ lists from scratch. Then, from Lemma~\ref{lem:lfd-dt-feature-efficiency} our result follows.

Now, we need to understand how adding the suffix $g_s$ from a target $g$ on which \LFD{} failed, makes the representation more useful. As a warm up, we can show that when the protocol faces the same target $g$ in the future, the updated representation  $\tildeF = \Fold \cup \{g_s\}$ will be able to learn it.  A crucial fact from which this follows is that \LFD{} Algorithm~\ref{alg:lfd-dt} {\em learns any list if and only if the list can be represented as a concatenation of prefixes of elements from $\tildeF$}.  This fact holds because Lemma~\ref{lem:lfd-correctness} and the way the algorithm works. Thus, since we were able to learn $g_p$ when we first saw $g$,  $g_p$ is a concatenation of prefixes from $\Fold$ i.e., $g_p \in \dtspace(\Pref(\Fold))$. Then, since $g  = (g_p, g_s) \in \dtspace(\Pref(\Fold \cup \{g_s\}))$, we can learn $g$ using $\tildeF$.

Of course, we should show that the updated representation is more powerful than just allowing us to learn repeated tasks in the future.  To see how, note that since the target $g$ is a concatenation of metafeatures from $\F$, its suffix $g_s$ must begin with the suffix of a metafeature from $\F$. More formally, since $g \in \dtspace(\F)$, $g_s$ must begin with a suffix $f_s$ of an element $f \in \F$. Let $f_p$ be the corresponding prefix of $f$.  Now, consider a future target that contains $f$. If the learner is able to identify all nodes in the target upto the end of prefix $f_p$, the learner is also guaranteed to identify $f$ completely in the target. This tells us a little bit more about the power of the updated representation.

 Now, to prove our lemma, we use the fact that each failure of \LFD{} Algorithm~\ref{alg:lfd-dt} must correspond to a specific element $f \in \F$ as seen above. That is, there must exist an $f = (f_p, f_s) \in \F$ such that $f \dlsubseteq g$ and furthermore, \LFD{}  was able to learn upto a prefix $f_p$ of $f$ after which it failed. We claim that there can only be $O(K)$ failures of \LFD{} that corresponds to a particular $f$ in this manner. From here, our lemma immediately follows. 
To prove this claim, we will categorize the failures of  \LFD{} corresponding to $f$ into two different cases and bound the number of failures in each case. Throughout the following discussion, we will simply use the term failure to denote failure of \LFD{}. \\ 

We will divide failures corresponding to $f$ based on whether $f_p$ can be represented as a concatenation of prefixes from $\Fold$ or not. If it can be, we show that it is easy to argue that in any future target there will not be a failure corresponding to $f$.  If not, we present a more involved argument to show that there can be at most $K$ failure events corresponding to a particular $f$. Then, the bound of $\bigo{K^2}$ on the total number of failures follows.\\

\textbf{Case 1}: For the first case we assume that $f_p \in   \dtspace(\Pref(\Fold))$.  Then, clearly, this is true for the new representation $\tildeF$ i.e., $f_p \in \dtspace(\Pref(\tildeF))$. Furthermore, since there is 
%Of course, no prefix of $f$ longer than $f_p$ can lie in $ \dtspace(\Pref(\Fold))$, as otherwise we would have learned a prefix longer than $g_p$ in $g$. 
 a new element $g_s$ with $f_s$ as its prefix, $f_s \in \Pref(\tildeF)$. This implies that $f \in \dtspace(\Pref(\tildeF))$. This means that we can henceforth learn an occurrence of $f$ in a new target if learning has been successful until the beginning of $f$ in that target. In other words,  there can never be another failure that corresponds to $f$. This case can hence occur only once. \\

\textbf{Case 2}: The second case corresponds to $f_p \notin  \dtspace(\Pref(\Fold))$.  We will now subdivide this case further based on another metafeature $f' \in \F$, a part of which lies in some hypothesized metafeature in $\Fold$ and was used to learn/match a part of $f$ in $g_p$.  We will fix $f'$ and argue that there can be at most two failure events characterized by $f$ and $f'$ during the lifelong learning protocol. Since there are only $K$ different $f'$, then for a fixed $f$, there can only be $2K$ failure events of this type, thus completing our proof.

%First of all, $g_p$ can be represented as concatenation of pieces of metafeatures from $\F$.

We begin by informally explaining how we choose $f'$ to classify a given failure event.  We first note that there are two ways in which $g_p$ can be represented in terms of the true metafeatures $\F$. The ``direct'' representation corresponds to the fact that $g \in \dtspace(\F)$. On the other hand, there is also an ``indirect'' representation:  since Algorithm~\ref{alg:lfd-dt} could learn the prefix $g_p$ using $\Fold$, $g_p$ can be represented as a sequence of prefixes from $\Fold$. Since each element in $\Fold$ is a part of older targets from $\dtspace(\F)$, we can represent this sequence of prefixes in terms of parts of true metafeatures (that are not necessarily prefix/suffix parts).

Now, let the root variable of $f$ be $i_f$.  There must be a unique element in the sequence of prefixes that contains $i_f$. {\em We let $f'$ be the metafeature in $\F$ that contributes to the last bit of this unique element in the above-described indirect representation}. %Note that $f'$ basically belongs to the indirect representation of $g_p$ that we described above. 
Before we proceed to describe this more formally, we note that this is all possible only because  $i_f$ indeed belongs to $f_p$. If it did not, it means $f_p$ is an empty string, which we have dealt with in Case 1.

We now state our choice of $f'$ more formally. Since we were able to learn $g_p$ using $\Fold$ we can write 
$g_p = (\sPref(\tilde{f}_{l_1}),  \sPref(\tilde{f}_{l_2}),  \hdots)$ for $\tilde{f}_{l_1},\tilde{f}_{l_2}, \hdots  \in \Fold$ where we use the notation $\sPref(\tilde{f})$ to denote a particular prefix of $\tilde{f}$. 
Let $\sPref(\tilde{f}_{l_r})$ be the unique element in the above sequence that contains $i_f$ (we use the index $r$ to denote that it contains the root).  
Like we stated before, since $\tilde{f}_{l_r}$ is also the suffix of some old target in $\dtspace(\F)$, $\tilde{f}_{l_r}$ must be made up of parts of true metafeatures $\F$. The same holds for $\sPref(\tilde{f}_{l_r})$ too. 
We will focus on the true metafeature that makes up the last bit of $\sPref(\tilde{f}_{l_r})$.  That is, let $f' \in \F$ be the metafeature that occurs in an older target, such that a non-empty suffix of $\sPref(\tilde{f}_{l_r})$ comes from $f'$ i.e.,  there exists suffix  $\sSuff(\sPref(\tilde{f}_{l_r}))$ such that $\sSuff(\sPref(\tilde{f}_{l_r})) \dlsubseteq f'$. Here, again $\sSuff(\tilde{f})$ is used to denote a particular suffix of $\tilde{f}$. Thus each failure event in this case can be characterized by a particular $f$ and $f'$. 

Note that $\sSuff(\sPref(\tilde{f}_{l_r}))$ need not necessarily be a suffix of $f'$  because $\tilde{f}_{l_r}$ may have stopped matching with $g$ somewhere in the middle of $f'$. It need not necessarily be a prefix of $f'$ either because $\tilde{f}_{l_r}$ is only a suffix of some target in $\dtspace(\F)$ and this suffix may have begun somewhere in the middle of $f'$ in that target.

 To show that there are at most two failure events for a given $f$ and $f'$, we will consider two sub-cases depending on
 whether $i_f \dlnotin \sSuff(\sPref(\tilde{f}_{l_r}))$. That is, when we use a part of $f'$ to learn $g_p$, we see whether we learn $i_f$ or not.
 These two cases are illustrated in Figure~\ref{fig:case21} and Figure~\ref{fig:case22}. For both these scenarios, we first analyze the structure behind the failure i.e., the locations of the different variables and how the different metafeatures align with each other.  Based  on this, we show that for each type, there can be at most one failure.\\

\textbf{Case 2a}:  $i_f \dlnotin \sSuff(\sPref(\tilde{f}_{l_r}))$.  Let us call this an $(f,f')^{(1)}$-type failure event. We first look at how the different elements are positioned when such a failure occurs, by aligning the elements in a way that the variables match. First, recall that by the definition of $\sPref(\tilde{f}_{l_r})$, $i_f \dlin \sPref(\tilde{f}_{l_r})$.  Thus, $i_f\dlin\tilde{f}_{l_r}$. Furthermore, by definition of $f'$, and because $\tilde{f}_{l_r}$ is the suffix of an older target from $\dtspace(\F)$, either a suffix  or the whole of $f'$ must occur in $ \tilde{f}_{l_r}$. We claim that  1) it is the latter, i.e., $f' \dlsubseteq \tilde{f}_{l_r}$ and furthermore, 2) the root of $f'$ is located below $i_f$ in $\tilde{f}_{l_r}$ (as illustrated in Figure~\ref{fig:case21}). If only a suffix of $f'$ occurred in $\tilde{f}_{l_r}$, it means that $\tilde{f}_{l_r}$ begins with that particular suffix and therefore by definition of   $\sSuff(\sPref(\tilde{f}_{l_r}))$ being the last bit of $\sPref(\tilde{f}_{l_r})$ that comes from $f'$, $\sPref(\tilde{f}_{l_r}) = \sSuff(\sPref(\tilde{f}_{l_r}))$. Then, since $i_f \dlin  \sPref(\tilde{f}_{l_r})$, $i_f \dlin  \sSuff(\sPref(\tilde{f}_{l_r}))$ which is a contradiction. Now, if indeed $f' \dlsubseteq \tilde{f}_{l_r}$ but the root of $f'$ was not located below $i_f$ in $\tilde{f}_{l_r}$ again by definition of $\sSuff(\sPref(\tilde{f}_{l_r}))$ being the last bit of $\sPref(\tilde{f}_{l_r})$ that comes from $f'$, $i_f \dlin  \sSuff(\sPref(\tilde{f}_{l_r}))$ which is a contradiction. Note that conclusions 1) and 2) above mean that $ \sSuff(\sPref(\tilde{f}_{l_r}))$ is a prefix of $f'$.

\begin{figure}
\centering
  \includegraphics[width=.9\linewidth]{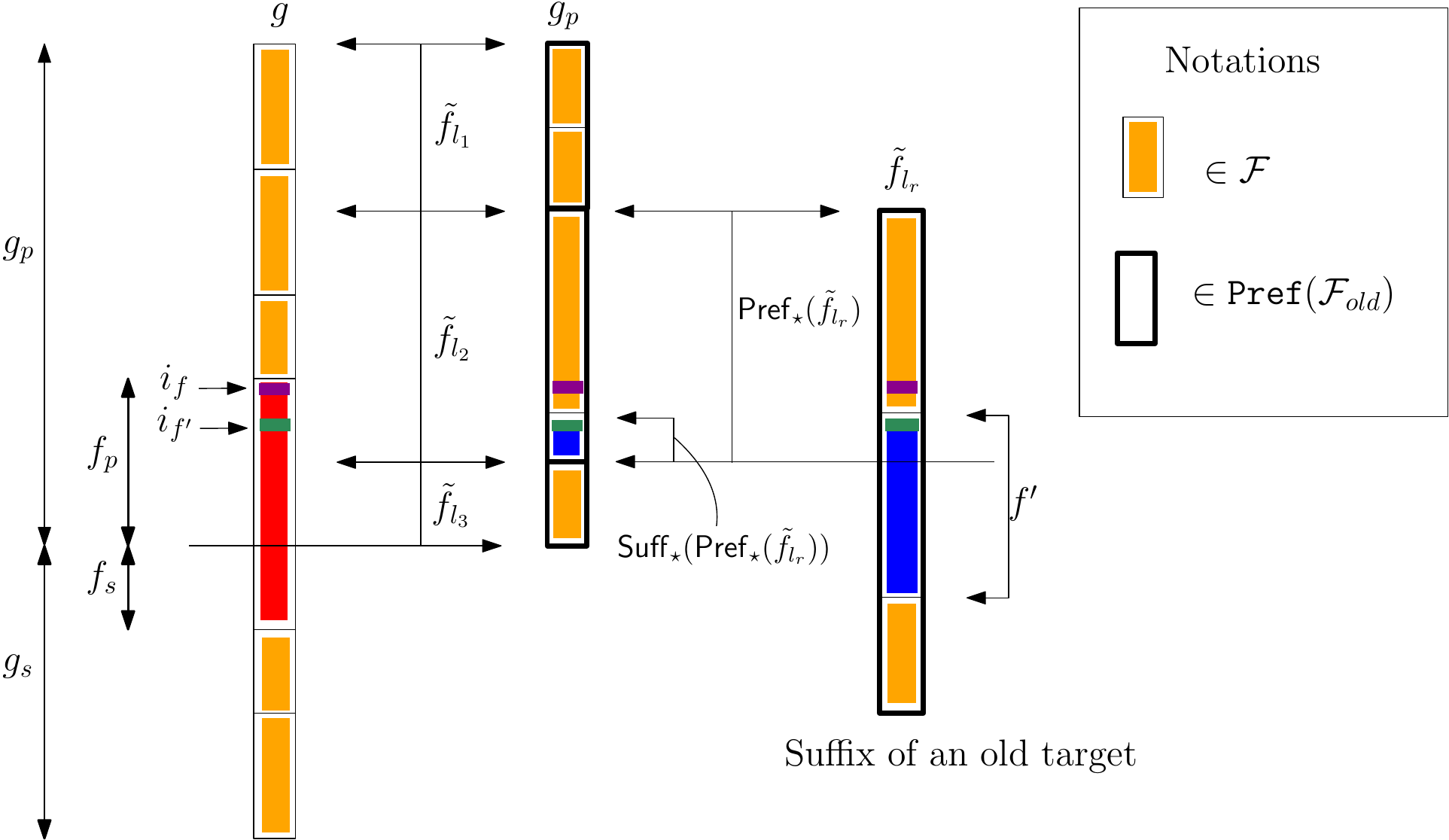}
\caption{\textbf{$(f,f')^{(1)}$-type failure where $i_f \dlnotin \sSuff(\sPref(\tilde{f}_{{l_r}}))$:} We represent the decision list $g$ on the left. Each subrectangle in this corresponds to some element from $\F$ with $f$ marked in red.  In the middle column, we represent the prefix of $g$, $g_p$ in terms of the elements of $\Pref(\Fold)$ each denoted by a thick subrectangle. We can do this because we were able to learn $g_p$ from $\tildeF$. Now each of these thick subrectangles can in turn be represented using parts of metafeatures from $\F$ because these are suffixes of actual targets. In particular, we choose the thick subrectangle that matched with the root of $f$ and show the complete metafeature from $\Fold$ on the right. In this metafeature, the thin rectangles correspond to its representation in $\F$. Observe that we have marked $f'$ in blue, and a part of it is what makes the last bit in the rectangle marked as $\sPref(f_{l_r})$ in $g_p$. Also $i_f$ is marked in magenta below which $i_{{f'}}$ is marked in green.}
\label{fig:case21}
\end{figure}

\begin{figure}
\centering
  \includegraphics[width=.9\linewidth]{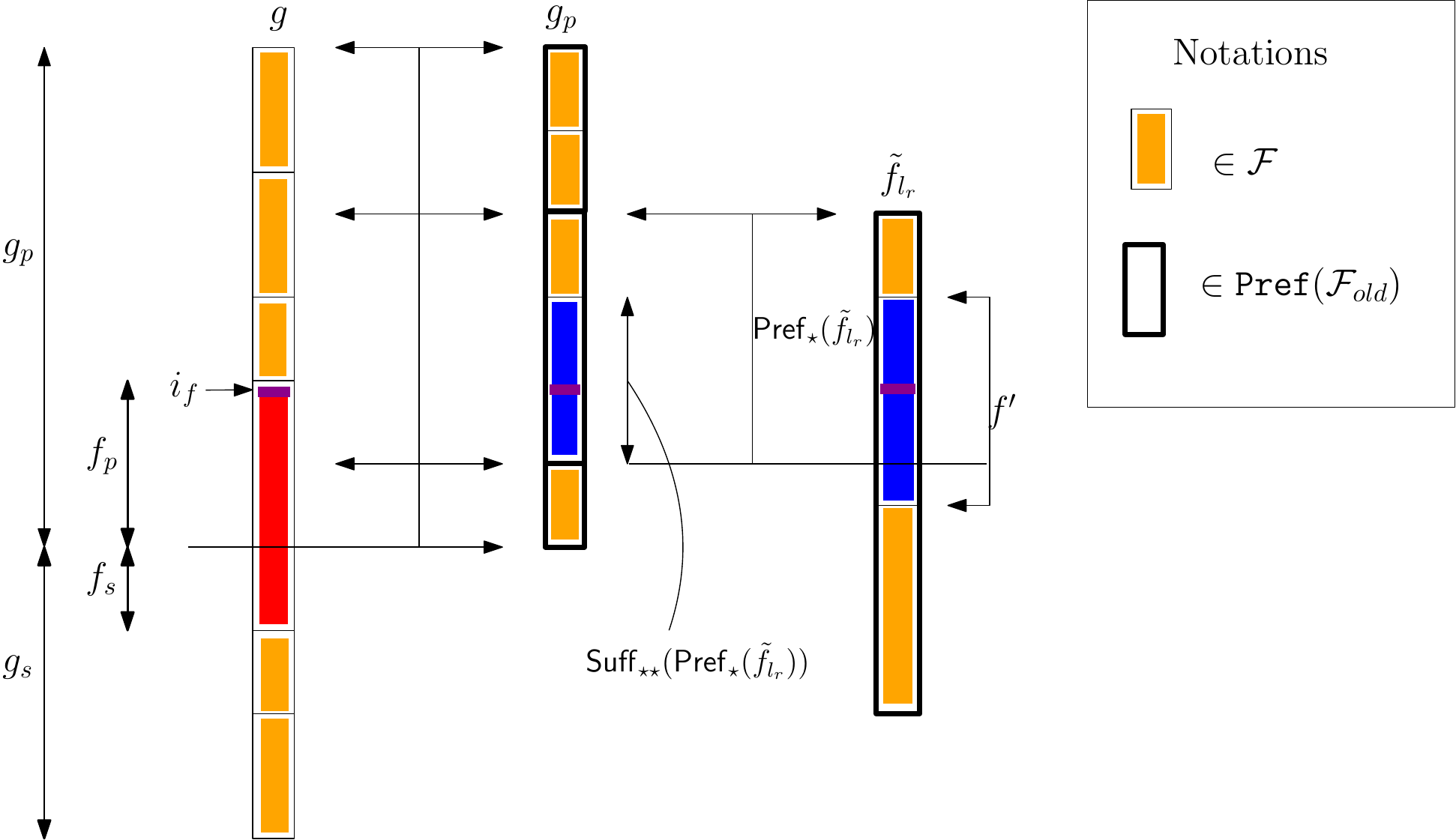}
\caption{\textbf{$(f,f')^{(2)}$-type failure where $i_f \dlin \sSuff(\sPref(\tilde{f}_{{l_r}}))$}}
\label{fig:case22}
\end{figure}

Given this, assume on the contrary that we do face a later target $g'$ with an $(f,f')^{(1)}$-type failure event. Then, we can define notations similar to the first failure.  Let $g_p'$ be the prefix we were able to learn correctly using $\tildeF$. Then, $g_p'$ can be similarly expressed as a sequence of prefixes from $\tildeF$, say $
(\sPref'(\tilde{f}_{l'_1}),  \sPref'(\tilde{f}_{l'_2}),  \hdots)$. By definition of this failure type, $f \dlsubseteq g'$. So consider the prefix that contains $i_f$, call it $\sPref'(\tilde{f}_{{l'_{{r'}}}})$. Furthermore, $\sPref'(\tilde{f}_{{l'_{{r'}}}})$ has a suffix $\sSuff'(\sPref'(\tilde{f}_{{l'_{{r'}}}}))$ that is also a part of $f'$ but is not necessarily the same as $\sSuff'(\sPref(\tilde{f}_{{l_{r}}}))$.  

We will now show that a prefix longer than $g_p'$ that includes $f$ completely can be represented using prefixes from $\tildeF$ which contradicts the fact that the algorithm failed somewhere in between $f$. To do this, we will make use of the fact that the algorithm was able to learn until $i_{f'}$ in the second failure, beyond which it can learn the rest of the target until the end of $f_p$ like it did the previous time, after which we can append $f_s$ from the representation. More specifically, observe that there is exactly one position at which $i_{f'}$ in $f'$ can match with $f$ and hence the failure will look similar to Figure~\ref{fig:case21} again; $f'$ will be contained in $\tilde{f}_{l'_{{r'}}}$ and $i_f$ will be located above $i_{f'}$. Now, since we also know that $f' \dlsubseteq \tilde{f}_{l'_{{r'}}}$, we can extend/shorten the prefix $\sPref'(\tilde{f}_{l'_{{r'}}})$ that is used to match with $g_p'$ to another prefix $\sPref''(\tilde{f}_{l'_{{r'}}})$ that has the same suffix as before, $\sSuff(\sPref(\tilde{f}_{l_{r}}))$.  On doing this, the rest of $f_p$  in $g'_p$ can be represented using the same prefixes from $\tildeF$ used to represent that part in $g_p$. Furthermore, we can append $f_s$ to this sequence because $f_s$ is a prefix of $g_s$ that was added to the representation. Thus, we take the sequence $
(\sPref'(\tilde{f}_{l'_1}),  \sPref'(\tilde{f}_{l'_2}),  \hdots)$
1) we retain the first ${r'}-1$ elements, 2) modify the ${r'}$th element so that its suffix matches with $\sSuff(\sPref(\tilde{f}_{l_{r}}))$, 3) append the $r$th, $r+1$th, $\hdots$ elements from the representation for $g_p$, 4) and finally append $f_s$. This represents a larger prefix of $g$ that includes $f$ completely, using only prefixes from $\tildeF$. Namely, this is $(\sPref'(\tilde{f}_{l'_1}),  \sPref'(\tilde{f}_{l'_2}), \hdots \sPref''(\tilde{f}_{l'_{{r'}}}),  \sPref(\tilde{f}_{l_{r+1}}), \sPref(\tilde{f}_{l_{r+2}}), \hdots, f_s)$. This contradicts the fact we failed to learn $f$ completely in $g'$.\\

\textbf{Case 2b}: $i_f \dlin \sSuff(\sPref(\tilde{f}_{{l_r}}))$.  Let us call this an $(f,f')^{(2)}$-type failure event. We now make  a similar argument. The only difference is that now $\sSuff(\sPref(\tilde{f}_{l_{r}}))$  is not necessarily a prefix of $f'$ and therefore, $i_{f'}$ is not necessarily present in $\sSuff(\sPref(\tilde{f}_{l_{r}}))$ (see Figure~\ref{fig:case22}. However it is guaranteed that a suffix of $f'$ containing $i_f$ is present in $\tilde{f}_{l_r}$. Now let $\ssSuff(\sPref(f_{l_r}))$ be an alternative shorter suffix of $\sPref(f_{l_r})$ that begins only at $i_f$. 
 
 Now, consider a new target with a similar failure with a similar $\ssSuff'(\sPref'(f_{l'_{r'}}))$ that begins with $i_f$. We will  again show how we can use the updated representation to represent a larger prefix of $g'$ , specifically a prefix that extends until the end of $f$ in $g'$. In particular, we make use of the fact that the algorithm was able to learn at least before $i_f$ in this target, beyond which we can learn $f_p$ the way we did in the previous target, and then append $f_s$ from the representation.
More specifically, we first  extend/shorten the prefix $\sPref'(f_{l'_{r'}})$ that is used to match with $g_p'$ to another prefix $\sPref''(f_{l'_{r'}})$ that it has the suffix $\ssSuff(\sPref(f_{l_{r}}))$ (which is only possible because $i_f \dlin \sPref''(f_{l'_{r'}})$).  On doing this, we can represent the rest of $f$ using $\tildeF$ like in the previous case. 

Thus, we take the sequence $
(\sPref'(\tilde{f}_{l'_1}),  \sPref'(\tilde{f}_{l'_2}),  \hdots)$
1) we retain the first ${r'}-1$ elements, 2) modify the ${r'}$th element, 3) append the $r$th, $r+1$th, $\hdots$ elements from the representation for $g_p$, 4) and finally append $f_s$. This represents a larger prefix of $g$ that includes $f$ completely, using only prefixes from $\tildeF$. Namely, this is $(\sPref'(\tilde{f}_{l'_1}),  \sPref'(\tilde{f}_{l'_2}), \hdots \sPref''(\tilde{f}_{l'_{{r'}}}),  \sPref(\tilde{f}_{l_{r+1}}), \sPref(\tilde{f}_{l_{r+2}}), \hdots, f_s)$. This contradicts the fact that we failed to learn $f$ completely in $g'$.

\end{proof}

%% file: monomials.tex
%Monomials can be considered to be real-valued analogs of boolean conjunctions. While a boolean conjunction would be of the form $x_{i_1} x_{i_2} \hdots x_{i_d}$  mapping an input $\x \in \{0,1\}^{N}$ to $\{0,1\}$

We consider lifelong learning of degree-$d$ monomials under the belief that there exists a set of $K$ monomial metafeatures like $\{ x_1 x_2, x_1^2x_3, \hdots \}$ and each target can be expressed as a product of powers of these metafeatures e.g., $(x_1x_2)^2 (x_1^2 x_3)$. This is similar to the lifelong Boolean monomial learning discussed in \citet{eff-rep} where each monomial is a conjunction of monomial metafeatures. Since that is an NP-hard problem, they assume that the metafeatures have so-called ``anchor'' variables unique to each. We will however not need this assumption.

Formally, for any input  $\x = (x_1, x_2, \hdots x_\N) \in \mathbb{R}^N$, we denote the output of a $d$-degree target monomial $\G = (g_1, g_2, \hdots, g_N)$ by the function
 $P_{\G}(\x) = x_1^{\g_1} x_2^{\g_2} \hdots x_N^{g_N}$ where $\g_i \in \mathbb{N} \cup \{ 0 \}$ and the {degree} $\sum_i \g_i \leq d$. The unknown metafeature set $\F = \{\f_1, \f_2, \hdots \f_K\}$ consists of $K$ monomials.
 % such that each target can be expressed as a product of powers of these metafeatures. 
 To simplify notations, we also consider $\F$ to be a matrix where column $i$ is  $\f_i$. Therefore, if $\G$ can be expressed using $\F$, then $\G$ lies in the column space of $\F$ denoted by $\C(\F)$. Then, our problem setup is as follows:

\begin{assumption}\label{ass:monomial-general} The $m$ $d$-degree targets $\G^{(1)}, \hdots \G^{(m)}$ and the training data (each of at most $S$ examples) drawn from unknown distributions $\distr^{(1)}, \hdots, \distr^{(2)}$ satisfy the following conditions:

\begin{enumerate}
\item There exists an unknown $N \times K$ matrix $\F$  ($K \ll N$) such that  $\G^{(j)} \in \C(\F)$. %as described in Definition~\ref{def:dict-gen-monomials}.
%\item Each monomial target has a degree of at most $d$.
\item %The samples for all tasks are drawn i.i.d from a common unknown distribution $\distr$ over $\mathbb{R}^N$.
% (as is assumed in \cite{eff-rep}). 
Each $\distr^{(j)}$ is a product distribution  (as assumed in  \citealp{eff-rep,andoni})
%(a crucial but oft-made assumption  \cite{eff-rep,andoni}) 
that is not too concentrated (explained in Appendix~\ref{app:pm}). 
%That is, the values $x_i$ are drawn independently for each sample. (This is a crucial but oft-made assumption  \cite{eff-rep,andoni})
\end{enumerate}
\end{assumption}

Unlike the decision tree problem, where we only considered an abstraction of the learning routine, here we present a particular technique for learning a monomial exactly. We show that under product distributions that are not too concentrated, it is possible to \emph{exactly} learn the power of a given feature in a target by examining {\em only that} feature on polynomially many samples (Lemma~\ref{lem:g-hat-g} in Appendix~\ref{app:pm}). Naturally, we can learn the monomial exactly from scratch as presented in Algorithm~\ref{alg:monomials-scratch} in Appendix~\ref{app:pm} from only polynomially many samples. Then, in the lifelong learning model, by merely keeping a record of the features that have been seen so far, it is fairly straightforward to
learn only $K$ targets from scratch while learning the rest by examining $\bigo{Kd}$ features per example (Theorem~\ref{thm:naive-monomials}).   % i.e.,  $\bigo{S(KN + mKd)}$ feature evaluations in total. 

Here, we present a significantly better protocol that learns only $K$ targets from scratch and on the rest, evaluates only $\bigo{K}$ features on all examples and $d$ features on one example. This is an improvement especially for cases where $d$ is large. 

The key idea is that we store a list of targets that have been learned from scratch as columns of the matrix $\tildeF$. Also, we learn a new target from scratch only if $g \notin \C(\tildeF)$. Therefore, after learning $K$ targets from scratch, we can show that we have a  rank-$K$ matrix $\tildeF$  such that $\C(\tildeF) = \C(\F)$. Therefore all future targets can be learned using $\tildeF$.

Now, the idea for \LFD{} is as follows (see Figure~\ref{fig:monomials} for illustration). If we have learned $k$ targets from scratch, then $\tildeF$ is of rank $k$. Then,  we identify a set of $k$ features $\I$ that correspond to linearly independent rows in $\tildeF$. We first learn only the powers of $\I$ (which we will denote by $\G[\I]$) by examining $\I$ on all samples. Then we learn a monomial $\tildeG$ by using the equation $\tildeG = \tildeF (\tildeF[\I])^{-1}\G[\I]$. If indeed $\G \in \C(\tildeF)$, then  $\tildeG $ equals $ \G$. This is because the power of each monomial metafeature in $\G$ is recovered through $(\tildeF[\I])^{-1}\G[\I]$.  However, we do not know if $\G \in \C(\tildeF)$ and it may be that $\tildeG\neq\G$. To address this, we can show using Lemma~\ref{lem:pit} that we only need to draw a single sample $\x$, examine $d$ features relevant to $\tildeG$ and check if our prediction $P_{\tildeG}(\x)$ equals the true label $P_{\G}(\x)$.
 If this fails, we conclude that $\tildeG \neq \G$ and therefore, $\G \notin \C(\tildeF)$. We learn $\G$ from scratch and add it to $\tildeF$. 
 Thus, \LFD{} examines only at most $K$ features on all but one sample and $d$ features on one final sample. In fact, after learning $K$ targets from scratch, we do not need to examine the $d$ features and do the verification step because we are guaranteed that $\G \in \C(\tildeF)$.

\begin{figure}
\centering
\includegraphics[scale=0.75]{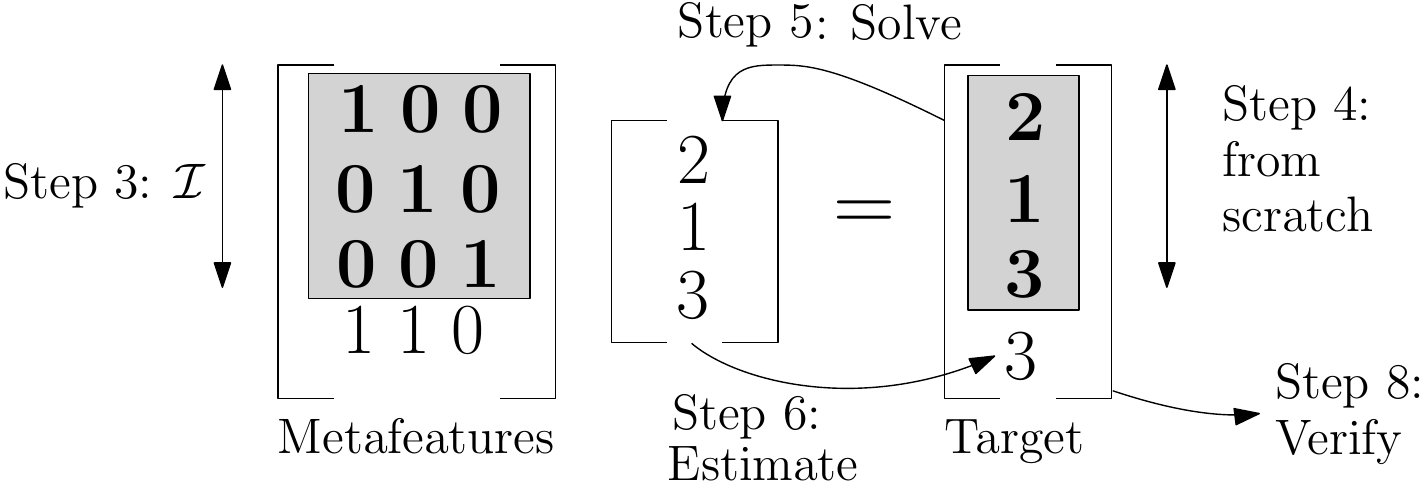}
\caption{\textbf{Illustration of \LFD{} Algorithm~\ref{alg:lfd-prod-mon}}}
\label{fig:monomials}
\end{figure}

\begin{algorithm}
\caption{\ID{} - Monomials}
\label{alg:id-prod-mon}
\begin{algorithmic}[1]
\STATE Input: Old representation $\Fold$ and $\G$ learned from scratch
\STATE Return $\tildeF = [\Fold, \G]$
\end{algorithmic}
\end{algorithm}

\begin{algorithm}
\caption{\LFD{} - Learning a Monomial from Metafeatures}
\label{alg:lfd-prod-mon}
\begin{algorithmic}[1]
\STATE Input: Metafeatures $\tildeF = [ \tilde{\f}_1, \hdots, \tilde{\f}_k ]$ ($k \leq K$),  sample set $\sample$ of size $\samplesize$. %$\bigo{ \frac{d}{\left(\min (\frac{c^2}{d}, \frac{c}{d}, 1)\right)^2}\log \frac{mK}{\delta} }$ .
\STATE Halt with failure if $\tildeF$ is empty.
\STATE Let $\I$ be the indices of those rows in $\tilde{\F}$ that are linearly independent and let  $\tildeF[\I]$ be the corresponding $k \times k$ sub-matrix of $\tildeF$. 
%\STATE Draw $\bigo{ \frac{d}{\left(\min (\frac{c^2}{d}, \frac{c}{d}, 1)\right)^2}\log \frac{mK}{\delta} }$ from $\distr$ and 
\STATE Examine features $\I$ on all samples and use Lemma~\ref{lem:g-hat-g} to learn and round off estimates $\tilde{\g}_{i}$ for each $i \in \I$.
\STATE Solve for $\w_{\tildeF[\I]}(\G[\I])$ in $\tildeF[\I] \w_{\tildeF[\I]}(\G[\I]) = \G[\I]$. If no solution exists, halt with failure.
\STATE Estimate $\tilde{\G} \gets \tildeF \w_{\tildeF[\I]}(\G[\I])$.
\STATE Halt with failure if the degree of $\tilde{\G}$ is greater than $d$.
\STATE \label{step:verify} Draw a single sample $(\x, P_{\G}(\x))$, examine the features relevant to $\tilde{\G}$. If $P_{\G}(\x) \neq P_{\tilde{\G}}(\x)$, halt with failure.
\STATE Return $\tilde{\G}$.
\end{algorithmic}
\end{algorithm}

\begin{restatable}{theorem}{monomials}
\label{thm:monomials}
%In the model of Problem Setup~\ref{ass:monomial-general} for monomials, 
The  (\LFD{} Algorithm~\ref{alg:lfd-prod-mon}, \ID{} Algorithm~\ref{alg:id-prod-mon})-protocol for monomials makes $\bigo{S(KN + mK)+md}$ feature evaluations overall and runs in time $poly(m,N,K,\samplesize,d)$. % $\left(K, K\right)$-\efficient{}, though the protocol also examines $d$ features on at most one sample per target (only on targets before the $K^{th}$ failure).  
\end{restatable}

\begin{proof}  
We show in Lemma~\ref{lem:pm-lfd} that with high probability $1 - \bigo{\frac{\delta}{m}}$ for any given target\footnote{It is only with high probability because the algorithm for learning the power of a particular features works correctly only with high probability.}, if we add a metafeature to $\tildeF$ then this increases the rank of $\tildeF$. 
Applying Lemma~\ref{lem:pm-lfd} over at most $m$ problems, this then holds over the whole sequence of $m$ problems. Assume we fail to learn from our representation on more than $K$ targets.	 This means that there will be at least $K+1$ targets (the columns of $\tildeF$) that are linearly independent. However,  since all targets belong to $\C(\F)$, there cannot be more than $K$ targets that are linearly independent. Thus, we achieve a contradiction. Now, since we learn only at most $K$ targets from scratch, applying Theorem~\ref{thm:pm-scratch} (from Appendix~\ref{app:pm}) over these we get that we learn them correctly with probability $1- \bigo{\delta}$. Also, since $\tildeF$ has at most $K$ columns, from Lemma~\ref{lem:pm-lfd} we have that each time we learn using the representation, we examine $K$ features per example. Besides, we examine $d$ features that are relevant to $\G$ in Step~\ref{step:verify}. 
\end{proof}

\begin{lemma}
\label{lem:pm-lfd}
Let $\tildeF$ be an $N \times k$ matrix. Then, with high probability $1 - \bigo{\frac{\delta}{m}}$, a) if $\G \in \C(\tildeF)$, then Algorithm~\ref{alg:lfd-prod-mon} correctly learns and outputs $\tilde{\G}=\G$ b) if Algorithm~\ref{alg:lfd-prod-mon} does output some $\tilde{\G}$, then $\tilde{\G}= \G$, c) Algorithm~\ref{alg:lfd-prod-mon} examines only at most $k$ features per sample point and at most $d$ features on a single sample.
\end{lemma}

\begin{proof}
a.  Given that $\tildeF$ is of rank $k$, then if $\G \in \C(\tildeF)$, there exists a unique solution for $\w_\tildeF(\G)$ in $\tildeF \w_\tildeF(\G) = \G$. Note that this is a system of $N$ linear equations in $k$.  Therefore, if the Algorithm picked any set of $k$ linearly independent rows $\I = \{i_1, i_2, \hdots i_k\}$ from $\tildeF$, there must exist a unique solution to $\tildeF[\I] \w_{\tildeF[\I]}(\G[\I]) = \G[\I]$ where the solution is $ \w_{\tildeF[\I]}(\G[\I]) =  \w_{\tildeF}(\G)$. Thus, solving this system will give us the value of $\w_{\tildeF}(\G)$ from which we can compute $\G$ correctly using $\tildeF \w_\tildeF(\G) = \G$. This however requires that we determine the values of $\g_{i_1}, \g_{i_2}, \hdots, \g_{i_k}$ from scratch, which we can do accurately with high probability of $1- \bigo{\frac{\delta}{m}}$ from  Lemma~\ref{lem:g-hat-g} (from Appendix~\ref{app:pm}) using polynomially many samples. 
%$\bigo{ \frac{d}{\left(\min (\frac{c^2}{d}, \frac{c}{d}, 1)\right)^2}\log \frac{mK}{\delta} }$ samples.

b.  To prove our second claim, observe that the only event in which the learner may potentially have an incorrect output is when $\G \notin \C(\tildeF)$ but we still do learn a $\w_{\tildeF[\I]}$ because it so happens that $\G[I] \in \C(\tildeF[\I])$. However, $\tilde{\G} = \tildeF\w_{\tildeF[\I]}(\G[\I]) \neq \G$. If $\tilde{g}$ has a degree greater than $d$, the algorithm halts with failure. Otherwise, we can show using Lemma~\ref{lem:pit} (see Appendix~\ref{app:pm}) that by drawing a single sample and checking whether $P_{\tilde{\G}}(\x) = P_{{\G}}(\x)$ we can conclude whether $\G = \tilde{\G}$.

c. This follows directly from the design of the algorithm: we examine only $K$ features on all samples, and then on a single new sample we examine features relevant to $\tildeG$ provided $\tildeG$ has degree at most $d$.\end{proof}

%% file: polynomials.tex
\subsection{Polynomials}

In this section, we study lifelong learning of real-valued polynomial targets each of which is a sum of at most $t$ degree-$d$ monomials. Similar to the Boolean model in \citet{eff-rep}, our belief is that there exists a set of monomial metafeatures such that each monomial in the polynomial can be expressed as a product of these metafeatures like we described in the previous section. As an example, given $\F = \{ x_1 x_2, x_1^2x_3, \hdots \}$, one possible target is $3(x_1 x_2)( x_1^2x_3) - 5(x_1 x_2)^2( x_1^2x_3) $. Again, we assume that each $\distr^{(j)}$ is a product distribution over $\mathbb{R}^N$. Since polynomial learning is a hard problem, we will have to make a strong assumption that 
each $\distr^{(j)}$ is {\em known}, which then enables us to adopt the polynomial learning technique from \cite{andoni}. Note that we can relax this assumption when all the distributions are common (like it is assumed in \citealp{eff-rep}), so that the common distribution can first be learned using $\bigo{poly(N)}$ feature evaluations on unlabeled examples. However, if the distributions were all different, learning them may need $\bigo{poly(mN)}$ feature evaluations, which would be feature-inefficient. 

Formally, for any input  $\x \in \mathbb{R}^N$, we denote the output of a $t$-sparse $d$-degree target polynomial $\setG = \{(\G_1, a_{\G_1} ), (\G_2, a_{\G_2} ),  \hdots  \}$   ($|\setG| \leq t$) by the function $P_{\setG}(\x) =\sum_{(\G, a_{\G}) \in \setG} a_{\G} P_{\G}(\x)$ where for each $(\G, a_{\G}) \in \setG$, $\G$ is a monomial of degree $d$ and co-efficient $a_{\G} \in \mathbb{R}$. Our belief is that  there exists a set of monomial metafeatures $\F$, and each polynomial can be represented as a sum of monomials, each of which can be represented using $\F$ as described in Section~\ref{sec:monomials}. More formally, a polynomial $\setG$ can be represented using $\F$ if for each $(\G, a_{\G}) \in \setG$, $\G \in \C(\F)$.  More compactly, $\setG^{(j)} \subset \C(\F) \times \mathbb{R}$. Then, our problem setup is as follows.

\begin{assumption}[\textbf{Lifelong polynomial learning}] \label{ass:polynomial-general} The $m$  $d$-degree $t$-sparse targets $\setG^{(j)}$ and data $\sample^{(j)}$  (each of at most $S$ examples) satisfy the following conditions:
\vspace{-6pt}
\begin{enumerate}
\item There exists an unknown $N\times K$ matrix $\F$ ($K \ll N$) such that each $\setG^{(j)} \in \C(\F) \times \mathbb{R}$.% as described in Definition~\ref{def:dict-gen-polynomials}.
%\item Each polynomial target is $t$-sparse (has at most $t$ terms) and has degree at most $d$.
% and a sparsity of $t$. That is, the polynomial is the sum of at most $t$ monomial terms each of degree $d$.
\item The samples in $\sample^{(j)}$ are drawn i.i.d from a \textbf{known} product distribution $\distr^{(j)}$ \footnote{This is the model considered in \cite{andoni}. An upper bound on $\samplesize$ can be found in \cite{andoni}.}. 
\end{enumerate}
\end{assumption}

% (slightly modified for our purposes).
\subsubsection{Learning a polynomial from scratch}

We now briefly discuss the algorithm in \citet{andoni} for learning a polynomial from scratch from a known distribution. The basic idea is to use correlations between the target and some cleverly chosen functions to detect the presence of different monomials in $\setG$. For the sake of convenience, assume there exist correlation oracles that when provided as input some function $P'$, return the exact value of the correlations $\langle P'(\x), P_{\setG}(x)\rangle$, $\langle P'(\x), P^2_{\setG}(x)\rangle$  etc., In practice these oracles can be replaced by approximate estimates based on the sample $\sample$. We will limit our analysis to the exact scenario noting that it can be extended to the sample-based approach in a manner similar to \cite{andoni}. Our guarantees will then hold good with high probability, given sufficiently many samples. 

To simplify the discussion we will assume like in \cite{andoni} that the distribution over each variable is identical i.e., $\distr = \mu^\N$.  Then, as a first step, given $\distr$, the learner creates an inventory of polynomials in each variable $x_i$ such that these polynomials represent an ``orthornormal bases'' with respect to $\distr$.  More formally, the inventory will consist of polynomials $H_{d'}(x_i)$ of degree $d'$ (identical for each $i \in [N]$) for each $0 \leq d' \leq d$, such that $\E[H_{d'}(x_i) H_{d''}(x_i)]$ is zero when $d' \neq d''$ and is one when $d' = d''$.  

Equipped with this machinery, we then set out to perform $t$ iterations extracting one monomial from $\setG$ at a time. Assume that from the iterations performed so far, we have extracted a set of monomials and their coefficients $\tilde{\setG} \subseteq \setG$. Now, for the next iteration, we first find the largest power of $x_1$ that is present in $\setG - \tilde{\setG}$ by testing whether $\langle H_{2d'}(x_1), (P_{\setG} - P_{\tilde{\setG}})^2  \rangle > 0$ for $d' = d, d-1, \hdots$ in that order. These tests detect the presence of $x_1^{d}$, $x_1^{d-1}, \hdots$ respectively. We stop when the test is positive for some $x_1^{d_1}$. The curious reader can refer \cite{andoni} to understand why this particular test works, but all we need to know for our discussion is that if these tests are done in this particular order, we are guaranteed to find the highest power of $x_1^{d_1}$ in $\setG - \tilde{\setG}$.  Then, we find the largest power of $x_2$ that ``co-occurs'' with $x_1^{d_1}$ in some monomial, by testing whether $\langle H_{2d_1}(x_1) H_{2d'}(x_2), (P_{\setG} - P_{\tilde{\setG}})^2  \rangle > 0$ for $d' = d, d-1, \hdots$ to detect  the presence of $x_1^{d_1}x_2^d, x_1^{d_1}x_2^{d-1}, \hdots$ and so on in that particular order. In this manner, the algorithm builds a monomial over $N$ sub-iterations which turns out to be the \textit{lexicographically largest} $\G$ present in $\setG - \tilde{\setG}$. Now, to compute the co-efficient $a_{\G}$ we find $\langle \prod_{i=1}^{N} ({b_{g_{i}}}{H_{g_i}(x_i)}),  P_{\setG} \rangle$ where $b_{g_i}$ is the co-efficient of $x_i^{g_i}$ in $H_{g_i}(x_i)$. The algorithm then adds $(a_{\G}, \G)$ to $\tilde{\setG}$ before proceeding to the next of $t$ iterations.

The above summary differs from that original algorithm presented in \cite{andoni} in the precise quantity that it extracts in each iteration. \cite{andoni} consider a representation of the polynomial in the orthornormal bases such that it is a weighted sum of terms of the form $H_{d_1}(x_1) H_{d_2}(x_2) \hdots H_{d_N}(x_N)$, and in each iteration they extract one such term. We however use the representation in the orthonormal bases only to detect the lexicographically largest monomial and its corresponding co-efficient and then remove the monomial itself.\\

\subsubsection{Lifelong Polynomial Learning} 
As a baseline in the lifelong learning model, we can learn the targets by making $\bigo{\samplesize(KN + mKd)}$ feature evaluations by simply remembering what features have been seen so far  (Theorem~\ref{thm:naive-poly} in Appendix~\ref{app:poly}).
% achieve a $(\bigo{Kd}, K)$-\efficiency{} (Theorem~\ref{thm:naive-poly}). 
Below, we present an approach %in Appendix~\ref{app:pm} 
that makes only $\bigo{\samplesize(KN + m(K+td))}$ feature evaluations. This is an improvement for sparse polynomials $t < K$ e.g., when $t = \bigo{1}$.

 The high level idea is to maintain a metafeature set of ``linearly independent monomials'' picked from previously seen targets, like we did in the previous section. When learning a target using $\tildeF$, we perform $t$ iterations to extract the monomials, but now in each iteration we find the lexicographically largest power restricted to at most $K$ features. 
These $K$ features, say $\I$, correspond to linearly independent rows in $\tildeF$. Given the powers of these features, say $\G[\I]$, we can determine powers of all the features using the representation like we did in the case of monomials. Then, as before, we extract $\G$ from the polynomial and proceed to the next iteration. After $t$ iterations, our estimate of the polynomial is complete, so we draw a single example to verify it. If our verification fails, we learn the polynomial from scratch and update the representation with more linearly independent monomials from the learned polynomial.

Note that the restricted lexicographic search examines only a fixed set of $K$ features per example. Besides this, in each of the $t$ iterations, we evaluate $d$ features relevant to the extracted monomial, accounting for $K+td$ feature evaluations per example.

\begin{algorithm}
\caption{\ID{} - Polynomials}
\label{alg:id-poly}
\begin{algorithmic}[1]
\STATE Input: Representation $\Fold$ and a target $\setG$ learned from scratch.
\STATE $\tildeF \gets \Fold$
\FOR{$\G \in \setG$}
	\STATE If $\G \notin \C(\tildeF)$, add $\G$ as a column to $\tildeF$.
\ENDFOR
\STATE Return $\tildeF$
\end{algorithmic}
\end{algorithm}

\begin{algorithm}
\caption{\LFD{} - Learning Polynomial from Metafeatures}
\label{alg:lfd-poly}
\begin{algorithmic}[1]
\STATE Input: Metafeatures $\tildeF = [ \tilde{\f}_1, \hdots, \tilde{\f}_k ]$ ($k \leq K$), distribution $\distr$
\STATE Halt with failure if $\tildeF$ is empty.
\STATE Let $\I$ be the indices of those rows in $\tilde{\F}$ that are linearly independent and let  $\tildeF[\I]$ be the corresponding $k \times k$ sub-matrix of $\tildeF$. 
\STATE Query for only the features $\I$ on all samples.
\STATE Initialize $\tilde{\setG}$ to be empty.
\FOR{$t$ iterations}
	 \STATE Let $\G$ be the lexicographically largest monomial  in $\setG - \tilde{\setG}$ with respect to $\I$. Find $\G[\I]$ using the lexicographic search technique from \cite{andoni} using the correlation oracle (in practice, estimate this using the $\sample$). 
	 \STATE Solve for $\w_{\tildeF[\I]}(\G[\I])$ in $\tildeF[\I] \w_{\tildeF[\I]}(\G[\I]) = \G[\I]$. If no solution exists, halt with failure.
    \STATE Estimate $\tilde{\G} \gets \tildeF \w_{\tildeF[\I]}(\G[\I])$.
	\STATE Halt with failure if the degree of $\tilde{\G}$ is greater than $d$.
	\STATE $a_{\tilde{\G}} \gets \langle  \prod_{i=1}^{N} (b_{g_{i}}H_{g_i}(x_i)),  (P_{\setG}-P_{\tilde{\setG}}) \rangle$ 
	%\STATE For each example $(\x, y) \in \sample$, query the features with non-zero exponents in $\tilde{\G}$ to compute $\sample' = \{(\x,y - P_{\tilde{\G}}(\x)) \;  | \; (\x,y) \in \sample \}$
	\STATE $\tilde{\setG} \gets  \tilde{\setG} \cup \{ \tilde{\G} \}$
\ENDFOR
\STATE Draw a single sample $(\x, P_{\setG}(\x))$ from $\distr$, query the $td$ features that are relevant to $\tilde{\setG}$. If   $P_{\setG}(\x)\neq P_{\tilde{\setG}}(\x)$, halt with failure.
\STATE Return $\tilde{\setG}$.
\end{algorithmic}
\end{algorithm}

\begin{restatable}{theorem}{poly}
\label{thm:poly}
% The lifelong learning protocol for polynomials in the model of Assumption~\ref{ass:polynomial-general} that uses \ID{} Algorithm~\ref{alg:id-poly} and \LFD{} Algorithm~\ref{alg:lfd-poly} achieves a feature-efficiency of $\bigo{K+dt}$ while learning at most $K$ targets from scratch.
 %In the model of Problem Setup~\ref{ass:polynomial-general} for polynomials,
 The  (\LFD{} Algorithm~\ref{alg:lfd-poly}, \ID{} Algorithm~\ref{alg:id-poly})-protocol for polynomials makes  $\bigo{S(KN + m(K+dt))}$ feature evaluations overall and runs in time $poly(m, N, K, \samplesize, t)$. %  $\left(K+dt, K\right)$-\efficient{}.
\end{restatable}

\begin{proof}
Below in Lemma~\ref{lem:poly-lfd}, we show that we increase the rank of $\tildeF$ by at least one every time we fail to learn using $\tildeF$ on some target. If \LFD{} has failed on more than $K$ targets it means that there are at least $K+1$ monomials from $\C(\F)$ that were added as columns to $\tildeF$ and are linearly independent. However,  since  $\C(\F)$ is a $K$-dimensional subspace in $\mathbb{R}^{N}$, this results in a contradiction, thus proving that at most $K$ failures of \LFD{} can occur. The result then follows from Lemma~\ref{lem:poly-lfd} and the fact that $|\tildeF|$ contains only at most $K$ targets.
\end{proof}

\begin{lemma}
\label{lem:poly-lfd}
Let $\tildeF$ be an $N \times k$ matrix. Then, a) if $\setG^{(j)} \in \C(\tildeF)$, then Algorithm~\ref{alg:lfd-poly} correctly learns and outputs $\tilde{\setG}^{(j)}=\setG^{(j)}$ b) if Algorithm~\ref{alg:lfd-poly} does output some $\tilde{\setG}^{(j)}$, then $\tilde{\setG}^{(j)} = \setG^{(j)}$. Also, Algorithm~\ref{alg:lfd-poly} examines only at most $k+td$ features per sample point.
\end{lemma}

\begin{proof}
a. Assume $\setG^{(j)} \in \C(\tildeF)$. The fact that in each iteration, we find the lexicographically largest value for the features $\I$ follows directly from the discussion in \cite{andoni}. However, we do have to prove that there is a unique $\G$ in $\setG$ such that $\G[\I]$ corresponds to the above value. This follows from the proof of Lemma~\ref{lem:pm-lfd} where we showed that for $\I$ corresponding to linearly independent rows, $ \w_{\tildeF[\I]}(\G[\I]) =  \w_{\tildeF}(\G)$ and hence given $\w_{\tildeF[\I]}(\G[\I])$ there is a unique $\G \in \C(\F)$ defined by $\G = \tildeF \w_{\tildeF[\I]}(\G[\I])$. 

Now, we need to prove that we find a co-efficient $a_{\tilde{\G}}$ for the to-be-extracted monomial, that satisfies $a_{\tilde{\G}} = a_{\G}$. We first note that $ \langle  \prod_{i=1}^{N} {H_{g_i}(x_i)},  (P_{\setG}-P_{\tilde{\setG}})  \rangle$ returns the co-efficient of $\prod_{i=1}^{N} {H_{g_i}(x_i)}$ in $(P_{\setG}-P_{\tilde{\setG}})$, say $a'_{\G}$, in the basis representation of the polynomial. Next, we claim that the co-efficient $a'_{\G}$ in the bases representation is contributed to purely by the co-efficient $a_{\g}$ in the monomial representation. If there was any other monomial that contributed to $a'_{\G}$, 
 then it had to have a lexicographically larger value than $\G$ with respect to $\I$ or equal to $\G$ with respect to $\I$. However, this contradicts the fact that $\G$ was chosen to be the unique lexicographically largest value with respect to $\I$. Thus, we only need to account for the contribution of the co-efficient of $\prod_{i=1}^{N} {H_{g_i}(x_i)}$ with an extra factor of $b_{g_i}$ which corresponds to the co-efficient of $x_i^{g_i}$ within $H_{g_i}(x_i)$.

b. This follows from the proof of Lemma~\ref{lem:pm-lfd} and Lemma~\ref{lem:pit} applied to polynomials.

c.  First of all, we examine $k$ features when we query $\I$ on all samples. Now, note that when we execute the algorithm using samples for the correlation oracles, we will have to compute $P_{\tilde{\setG}}(\x)$ on each sample $\x$. This however will only require evaluation of features relevant to $\tilde{\setG}$. Since $\setG$ consists of at most $t$ monomials each of degree at most $d$, this can be only as large as $td$. 
 \end{proof}

\noindent
\textbf{Sample-based estimation}: We note that when we replace the oracles by estimation using random samples,  we should be careful about approximation errors that may affect the lifelong learning protocol. For example, if we were to infer that a monomial term exists in $\setG$, when in reality it does not, we may incorrectly add it to our representation $\tildeF$ when it should not be. However, if the co-efficients of each term in the polynomial were not too small, we can overcome this problem by learning the co-efficient of the monomial, and checking whether it is above a small threshold, before deducing that it indeed is a term in the polynomial.

%% file: agnostic.tex
We propose a novel agnostic lifelong learning model where the learner faces $m+r$ learning tasks of which $m$ tasks are guaranteed to be related through the $K$ metafeatures in $\F$ while the other $r$ tasks are arbitrary.
%and then describe simple ideas to modify the above approaches to handle the agnostic case. We note that our model 
Note that this is different from the conventional sense of agnostic learning where each individual task may involve model misspecification or noisy labels. What makes this challenging is that the $r$ ``bad'' targets can be chosen and placed adversarially in the stream of tasks. Since in the worst case there is no hope of minimizing feature evaluations done on the bad targets, we adopt the natural goal of reducing the feature evaluations on the training data of the $m$ good targets. 

%Let us recall the scenario in the agnostic case:

\begin{assumption}
\label{ass:agnostic}
In the agnostic model, the learner is faced with a series of $m+r$ targets such that:
\begin{enumerate}
\item $m$ (good) targets are guaranteed to be related to each other through a set of at most $K$ metafeatures, while the remaining  $r$ (bad) targets can be adversarially chosen and placed.
\item the learner has to reduce the feature evaluations done on the samples for the $m$ related targets.
\end{enumerate}
\end{assumption}

We focus our discussion on learning decision trees with depth $d = \bigo{1}$ noting that it is straightforward to extend it to learning more general decision trees and to  other targets. In fact, in the following discussion, it may be helpful to imagine the targets to be decision stumps over just one feature and the metafeature set $\tildeF$ to simply be a set of $K$ features. Now, recall that in the original setup, $\tildeF$ consisted of $\bigo{K}$ useful metafeatures from at most $K$ targets that were learned from scratch \LFD{} failed to learn them. A problem that arises now is that $\tildeF$ may have been updated with metafeatures from bad targets. Then, even if $\tildeF$ contained $K$ metafeatures, we cannot guarantee that future good targets can be learned using $\tildeF$. What should we do then? 
 
To address this, we present two simple computationally-efficient solutions below 
%(described in detail in Appendix~\ref{app:agnostic})  
that highlight an interesting trade-off between the number of targets learned from scratch and the number of features evaluated on the remaining targets. In the {\em $r$-expansion} technique, we allow the learner to update $\tildeF$ on every failure of \LFD{} allowing the representation to get as large as it can. In the {\em $r$-restart} technique, we restrict the size of the representation but however, whenever the representation is ``bad'', we erase and start learning the representation all over again. \\%Finally, we describe a combined technique that deals with the trade off carefully and does better than both the above.

\subparagraph{$r$-expansion technique} Observe that since $m$ targets belong to $\dtspace(\F)$, there exists a representation of at most $\bigo{K+r}$ metafeatures that is sufficient to describe all the $m+r$ targets: a representation that is the union of $\tildeF$ and the $r$ bad targets as they are.  Thus, we allow the lifelong learner to update $\tildeF$ whenever its $\LFD{}$ fails, which would result in a representation of at most $\bigo{K+r}$ metafeatures. \LFD{} will fail on at most $K$ good targets (and possibly on all the $r$ bad targets which we do not care about) and learn the rest successfully evaluating $\bigo{K+r}$ features per example. Note that this protocol is essentially identical to the original protocol in Algorithm~\ref{alg:generic}. 

% i.e., $\bigo{\samplesize(KN + m(K+r))}$ feature evaluations overall. 

 \subparagraph{$r$-restart technique}  Alternatively, we enforce $|\tildeF| \leq K$ as before but when \LFD{} fails on a $K+1^{th}$ target, we learn that target from scratch after which we simply erase $\tildeF$ and effectively restart our lifelong learning from the next task. Every time \LFD{} fails on a $K+1$th target after the most recent restart, we restart similarly.  This technique learns more targets from scratch, $\bigo{rK}$ targets in particular, but evaluates only $\bigo{K}$ features per example on the remaining targets. The protocol is described more formally in Algorithm~\ref{alg:restart}. \\

\begin{algorithm}[H]
\caption{$r$-restart based $(\mathcal{A}_{\sf UR}, \mathcal{A}_{\sf IR})$-protocol for  agnostic lifelong learning in the model of Problem Setup~\ref{ass:agnostic}}
\label{alg:restart}
\begin{algorithmic}[1]
\STATE \textbf{Input}: A sequence of $m+r$ training sets $\sample^{(1)}, \sample^{(2)}, \hdots, $
% learning tasks 
corresponding to targets $g^{(1)}, g^{(2)}, \hdots $, $m$ of which can be represented using an unknown set $\F$ of $K$ metafeatures.
\STATE Let $\tildeF$ be our current learned representation. Initialize $\tildeF$ to be empty.
\FOR{$j=1,2, \hdots m+r$}
	\STATE Using $\tildeF$ and $S^{(j)}$, attempt to cheaply learn $g^{(j)}$ with %Algorithm 
\LFD{} algorithm $\mathcal{A}_{\sf UR}$.
	\IF{learning was not successful}
		\STATE Extract all features in $\sample^{(j)}$ and learn $g^{(j)}$ from scratch.
		\STATE If {$|\tildeF| = K$}, assign an empty representation to $\tildeF$.
		\STATE Provide $\tildeF$ and  $g^{(j)}$ as input to \ID{}  algorithm $\mathcal{A}_{\sf IR}$ to update $\tildeF$.
	\ENDIF
\ENDFOR
\end{algorithmic}
\end{algorithm}

When $r = \bigo{\max\left( \frac{m}{N}, \frac{KN}{m}, K \right)}$, it is easy to see that one of these two techniques makes only $\bigo{\samplesize(KN + mK)}$ feature evaluations, which is as good as the performance when $r=0$. To deal with larger values of $r$, we describe a combined technique that deals with the trade off carefully and does better than both the above:

\begin{restatable}{theorem}{agnostic}
 \label{thm:agnostic}
In the agnostic model where we face $m+r$ decision tree targets such that $m$ trees belong to $\dtspace(\F)$, the number of feature evaluations on the training data for the $m$ trees:
\vspace{-5pt}
\begin{itemize}
\item  the $r$-expansion technique is $\bigo{\samplesize(KN + m(K+r))}$.
 \item  the $r$-restart technique is $\bigo{\samplesize(rKN+mK)}$.
 \item a combination of $c$-expansion and $r/c$-restart  is  $O(\samplesize(\sqrt{rKNm} + Km))$, for $c = \sqrt{rKN/m}$ provided $r = \bigomega{ \max\left( m/n, KN/m, K  \right) }$.
\end{itemize}
%and when  feature evaluations on the training data for the $m$ trees.
 \end{restatable}

\begin{proof}
In $r$-expansion, we allow $\tildeF$ to have as many as $\bigo{K+r}$ metafeatures. Now, every bad target may result in adding $\bigo{1}$ metafeatures to $\tildeF$ while the $m$ bad targets will result in adding $\bigo{K}$ metafeatures to $\tildeF$. Thus, we will be able to learn all but $m$
good targets using $\tildeF$ by examining only $\bigo{K+r}$ features per example i.e., $\bigo{\samplesize(rKN + mK)}$ features overall.

In $r$-restart, every time \LFD{} fails on a $K+1$th target, we learn that target from scratch and then erase $\tildeF$ effectively restarting our lifelong learning. Now, at least one of the $K+1$ trees learned from scratch must be a bad target. This is because if none of the $K$ trees that were used to update $\tildeF$ were bad, $\tildeF$ would have been rich enough to represent all the good targets. This means that the $K+1$th target has to be a bad target. Thus, every restart corresponds to a failure of \LFD{} on at least one bad target and at most $K$ good targets. Then, we will face at most $r$ such restarts, learning at most $rK$ targets from scratch during the process and the rest from only $\bigo{K}$ features per example i.e., $\bigo{\samplesize(KN + m(K+r))}$ features overall.

Now when $r = \bigo{\max\left(\frac{KN}{m}, K \right)}$ observe that $r$-expansion makes only $\bigo{\samplesize(KN+mK)}$ feature evaluations. Similarly, when $r = \bigo{\frac{m}{N}}$, $r$-restart makes $\bigo{\samplesize(KN+mK)}$ feature evaluations. This is as good as our performance when $r=0$.

To deal with $r =  \bigomega{ \max\left( \frac{m}{N}, \frac{KN}{m}, K \right)}$, we can combine the above techniques, in particular, we combine $\frac{r}{c}$-restart with $c$-expansion. That is, between every restart we allow $\tildeF$ to accommodate $\bigo{K+c}$ metafeatures and when \LFD{} fails on the $K+c+1$th target we restart the representation. Recall that each bad target may contribute $\bigo{1}$ metafeatures while all the good targets contribute to $\bigo{K}$ metafeatures. Thus, between every restart \LFD{} would have failed on at most $K$ good targets and at least $c+1$ bad targets. Since there are only $r$ bad targets, we then face only $\bigo{\frac{r}{c}}$ restarts. Since we learn only $\bigo{\frac{r}{c}} \cdot K$ targets from scratch and learn the rest by examining only $\bigo{K+c}$ features per example, we evaluate $\bigo{\samplesize(\frac{r}{c}KN + m(K+c)})$ features overall.

The value of $c$ that optimizes the above bound is $c^* = \sqrt{\frac{rKN}{m}}$ and the minimum is $\bigo{\samplesize(\sqrt{rKNm} + mK)}$. But note that $c^*$ must take a meaningful value for this bound to hold good. That is, for $c$-expansion to make sense, we need $c^* \geq 1$ and for $\frac{r}{c^*}$-restart to make sense, $\frac{r}{c^*} \geq 1$. That is, we need $c^* \in [1,r]$, which can be verified to hold good when $r =  \bigomega{ \max\left( \frac{m}{N}, \frac{KN}{m}, K  \right)}$.
\end{proof}

\section{Lower bounds}
\label{sec:lb}

We prove lower bounds on the performance of any lifelong learner under different ranges of $r$ in the agnostic model.  In particular, we prove tight lower bounds for sufficiently small and large values of $r$, ignoring other problem-specific parameters and the sample size parameter $\samplesize$ (that scaled only logarithmically with $N$ for most of our target classes). An interesting insight here is that when $r$ is too large,
%$ \geq \min \left( \frac{mN}{K} , \frac{(N-K)^2 m}{K N}\right)$, 
 we prove that no learner is guaranteed to succeed by making $\bigo{mN}$ feature queries, which means that lifelong learning is no longer meaningful for really large values of $r$. 
 %We state our results formally below, with more discussion in Appendix~\ref{app:lb}.

Our main idea is a randomized adversary that poses decision stumps (trees with only the root node) or degree-1 monomials to the learner. In particular, we use Lemma~\ref{lem:adversary} where we show that when the adversary picks one feature at random from a pool of $N'$ features to be the decision stump/monomial, if the learner examines only $o(N')$ features, the learner will fail to identify the correct feature for the target with probability $\bigomega{1}$. Thus, for the learner to successfully complete the task, it must examine $\bigomega{N'}$ features. Then to force a learner to examine $\bigo{KN + mK}$ features, the adversary picks $K$ distinct features at random from the pool of $N$ features for the first $K$ targets. Then it assigns these $K$ features as the metafeatures  and picks the remaining targets at random from this chosen set of $K$ features. 
%A more detailed proof for this and the other cases can be found in Appendix~\ref{app:lb}.

\begin{restatable}{theorem}{lb}
\label{thm:lb}
Let $r_{\min} = \max \left( \frac{ m}{N}, \frac{K N}{m} , K \right)$,  $r_{\max}  = \min \left( \frac{mN}{K} , \frac{(N-K)^2 m}{K N}\right)$. In the agnostic model of Section~\ref{sec:agnostic}, 
%where the learner faces $m+r$ trees, 
there exists an adversary such that, on the $m$ good trees, any lifelong learner makes:
\vspace{-4pt}
\begin{itemize}
\item $\bigomega{NK + Km}$ feature evaluations when $0 \leq r \leq r_{\min}$.
\item $\bigomega{\max\left( \frac{r}{N-K}, 1 \right) KN + Km}$  feature evaluations when $r_{\min} \leq r \leq r_{\max}$.
\item $\bigomega{mN}$  feature evaluations when $r_{\max} \leq r$.
\end{itemize}
\end{restatable}

\begin{proof}
In Lemma~\ref{lem:adversary} we design our randomized adversary.
We prove Theorem~\ref{thm:lb} in the following three lemmas one for each range of $r$. 
First in Lemma~\ref{lem:realizable} we prove a lower bound of $\bigomega{KN + mK}$ that holds for any value of $r$. Then in Lemma~\ref{lem:intermediate} we prove a lower bound for intermediate values of $r$ and finally in Lemma~\ref{lem:large}, we prove a lower bound for large values of $r$. 
\end{proof}

\begin{lemma}
\label{lem:adversary} \textbf{(Randomized adversary)}
For a particular task, if the adversary picks a feature from a pool of $N'$ features ($N'\leq N$) to pose a single-feature target\footnote{It does not matter if the learner knows these $N'$ features or not.}, if the learner examines only $o(N')$ features, the learner will fail (i.e., pick the wrong feature) with probability $\bigomega{1}$. 
\end{lemma}

\begin{proof} Let $i^*$ be the feature chosen by the adversary at random from a pool of $N'$ features $\I^*$, and $\I$ be the set of features examined by the learner.  The random choice of $i^*$ corresponds to different possible outcome events. But observe that from the perspective of the learner the events corresponding to $i^* \notin \I$  (the adversary picking a feature not examined by the learner) are all indistinguishable. This crucial observation tells us that in all such events, the learner will adopt the same strategy. Let $Pr_l(i)$ denote the probability that the learner outputs feature $i$ in this strategy. Let $Pr_a(i)$ denote the probability that the adversary chose feature $i$ at random from its pool of $N'$ features.    

Then,  the probability that the learner fails is at least the sum of probability of the event that the adversary picks an $i$ from $\I^* - \I $ and the learner does not pick $i$. We lower bound this probability $\sum_{i \in\I^* - \I }Pr_a(i)(1-Pr_l(i))$ as follows:
  
  \[
\begin{array}{rcl}  
  \displaystyle\sum_{i \in\I^* - \I } \underbrace{Pr_a(i)}_{\frac{1}{N'}} (1-Pr_l(i)) &=& \displaystyle \frac{1}{N'} \sum_{i \in \I^* - \I }  (1-Pr_l(i)) \geq \displaystyle \frac{1}{N'} \left( |\I^* - \I|   -  \sum_{i \in \I^* - \I }  Pr_l(i)  \right)\\
  &&\\
    &\geq & \displaystyle \frac{1}{N'} \left( N' - o(N)   - 1  \right) = \bigomega{1}\\
  \end{array}
  \]
 The second inequality follows from the fact that $ \sum_{i \in \I^* - \I }  Pr_l(i) \leq 1$ and the number of examined features $|\I| = \bigo{N'}$.

\end{proof}

%Now we prove Theorem~\ref{thm:lb} in the following three lemmas one for each range of $r$. 
%First in Lemma~\ref{lem:realizable} we prove a lower bound of $\bigomega{KN + mK}$ that holds for any value of $r$. Then in Lemma~\ref{lem:intermediate} we prove a lower bound for intermediate values of $r$ and finally in Lemma~\ref{lem:large}, we prove a lower bound for large values of $r$. 

\begin{lemma} 
\label{lem:realizable}
There exists an adversary such that any lifelong learning algorithm makes $\bigomega{KN + mK}$ feature evaluations.
\end{lemma} 
 
 \begin{proof}
 For the first $K$ single-feature targets, our adversary randomly picks $K$ distinct features which will be the metafeatures. Each of the remaining $m-K$ tasks are targets that correspond to one of these $K$ chosen features at random.  Now note that for a task $j$ where $j \leq K$, the adversary effectively picks a feature at random from a pool of $N-j+1$ features (which excludes the $j-1$ features already chosen). Thus, the learner has to examine $\bigomega{N-j+1}$ features in order to not fail in this task with probability $\bigomega{1}$. Thus, over the first $K$ tasks, the learner has to examine $\bigo{\sum_{j=1}^{K} N-j+1} = \bigomega{KN}$  features over all. Then, in each of the following $m-K$ tasks, the learner has to examine $\bigomega{K}$ features per task i.e., $\bigomega{(m-K)K}$ features overall, which is $\bigomega{mK}$ since $m$ is large. 
 
 \end{proof}
 
Now we prove a better bound for values of $r$ greater than $r_{\min}=\max \left( \frac{ m}{N}, \frac{K N}{m} , K \right)$ but less than $r_{\max}  = \min \left( \frac{mN}{K} , \frac{(N-K)^2 m}{K N}\right)$. Here, instead of precisely choosing $m$ good targets and $r$ targets, the adversary will pose a set of targets and then choose $K$ features to be the metafeatures. We then show that $\bigtheta{m}$ of the targets are good targets and $\bigtheta{r}$ targets are bad targets that correspond to the remaining $N-K$ features.

 \begin{lemma}\textbf{(Lower bound for intermediate values of $r$)}
 \label{lem:intermediate}
When $r \leq r_{\max}$, there exists an adversary such that any lifelong learning algorithm makes $\bigomega{\max\left( \frac{r}{N-K}, 1 \right) KN + Km}$ feature evaluations.
 \end{lemma}

 \begin{proof}
 When $\frac{r}{N-K} \leq 1$, the lower bound of $\bigomega{KN + Km}$ follows from Lemma~\ref{lem:realizable}. Hence, consider $\frac{r}{N-K} > 1$. Let $m' = \frac{rN}{(N-K)}$. The adversary first presents $m'$ single-feature targets picked at random from the pool of all $N$ features. Then the adversary chooses $K$ random features to be the metafeatures, hence marking targets corresponding to these $K$ features as good targets, and the rest as bad. 
 
 Now, we can show that there are in fact $\bigtheta{m}$ good targets and $\bigtheta{r}$ bad targets, thus ensuring that this is a legal sequence of adversarial targets. Since $m' = \frac{r}{N-K}N \geq N$, using Chernoff bounds, with high probability $1- O(1)$, we have $\bigtheta{m' \frac{N-K}{N}} = \bigtheta{r}$ bad targets and $\bigtheta{m' \frac{K}{N}} = \bigtheta{\frac{rK}{(N-K)}}$ good targets. Since, $r \leq \frac{(N-K)^2 m}{K N}$, this translates to $\bigtheta{\frac{(N-K)m}{N}} = \bigo{m}$ good targets. Thus, this is a valid sequence of targets. 
 
 Now, from Lemma~\ref{lem:adversary}, we get that the learner has to evaluate $\bigomega{\frac{rK}{(N-K)} \cdot N}$ features overall. In addition to this, the adversary presents a sequence of $m$ good targets chosen at random from the $K$ metafeatures. Note that this is legal because we still pose only $\theta{m}$ good targets. This accounts for  $\bigomega{mK}$ more feature evaluations.

In total, the learner examines $\bigomega{\frac{rK}{(N-K)} \cdot N + mK}$ features.
 \end{proof}
 
 We finally show that for sufficiently large $r$ i.e., $r \geq r_{\max}$ and $r \geq r_{\min}$, the learner has to evaluate $\bigomega{mN}$ features. 
 
 \begin{theorem}\textbf{(For large $r$)}
 \label{lem:large}
 Given  $r \geq r_{\max}$ and  
 $r \geq r_{\min}$,  there exists an adversary such that any lifelong learning algorithm makes $\bigomega{mN }$ feature evaluations.
 \end{theorem}

 \begin{proof} The range of values of $r$ such that $r \geq r_{\max} = \min \left( \frac{mN}{K} , \frac{(N-K)^2 m}{K N}\right)$ can be split into the interval $r \geq  \frac{mN}{K} $ and the interval 
 $ \frac{(N-K)^2 m}{K N} \leq r  \frac{mN}{K}$. We will consider these two intervals separately and provide adversarial strategies for both. 
% We will first present a strategy for $r \geq  \frac{mN}{K} $. Then, we consider the case $r \geq \frac{(N-K)^2 m}{K N}$, but assuming that $r < \frac{mN}{K} $ because the other case has already been taken care of. 

\textbf{Case 1}: $r \geq  \frac{mN}{K} $. Let $m' = \frac{mN}{K}$. The adversary poses $m'$ targets to the learner chosen at random from all the $N$ features. Thus, the learner is forced to examine $\bigomega{N}$ features on each target. Then, the adversary chooses $K$ features to be good features, thereby marking some of the targets as good targets. We show that, of the $m'$ targets, there are $\bigtheta{m}$ good targets and only $\bigo{r}$ bad targets. Therefore, this is a valid sequence of targets and furthermore, on this sequence the learner examines $\bigomega{m \cdot N}$ features. 

To count the number of good targets, we observe that $m' = \bigomega{\frac{N}{K}}$. Then from Chernoff bounds, with high probability $1- O(1)$, we have that $\bigtheta{m' \frac{K}{N}}$ i.e., $\bigtheta{m}$ targets are good targets. Since $m' \leq r$, we have only $\bigo{r}$ bad targets. 

\textbf{Case 2}: $r <  \frac{mN}{K}, r \geq \frac{(N-K)^2 m}{K N}$. Now, we set  $m' = \sqrt{\frac{rNm}{K}}$ and sample $m'$ targets at random from the pool of all $N$ features. Then we pick $K$ random features to be the metafeatures and then present $m$ good targets choosing randomly from the pool of $K$ metafeatures.  

To count the number of good targets in the first sequence of $m'$ targets, observe that $m' \geq N$ because $r \geq \frac{KN}{m}$. Hence, with high probability $1- O(1)$, the number of good targets is $\bigtheta{m' \frac{K}{N}} = \bigtheta{\sqrt{\frac{rKm}{N}}}$. Since $r \leq \frac{mN}{K}$, this is $O(m)$. Similarly, with high probability $1-O(1)$, the number of bad targets is $\bigtheta{m' \frac{N-K}{N}} = \bigtheta{\sqrt{\frac{rNm}{K}} \cdot \frac{N-K}{N}} = \bigtheta{\sqrt{r} \cdot \sqrt{\frac{(N-K)^2 m}{KN}}}$. Then using  the inequality $r \geq \frac{(N-K)^2 m}{K N}$, we get that the number of bad targets is $O(r)$. Thus, this is a valid sequence of targets. Furthermore, on $\bigtheta{\sqrt{\frac{rKm}{N}}}$ good targets, the learner is forced to examine $\bigomega{N}$ features. Thus, on the first sequence the learner examines $\bigomega{\sqrt{rKmN}}$ features overall. Since $r \geq \frac{(N-K)^2 m}{K N}$, this is $\bigomega{m(N-K)}$. On the second sequence the learner examines $\bigo{mK}$ features overall. In total, this is $\bigomega{mN}$ feature evaluations.

 \end{proof}

%% file: dt-appendix.tex
We first present proofs from Section~\ref{sec:dt}. 
Then, in Appendix~\ref{app:relaxed}, we present results for more models of decision trees.

\subsection{Proofs from Section~\ref{sec:dt}}

Now, we present our baseline  lifelong learning algorithm that simplyremembers features that have been seen as metafeatures in its learned representation.

\begin{theorem}[\textbf{Naive lifelong learning of decision trees}]
\label{thm:dt-baseline}
There exists a naive lifelong learning protocol for decision trees in the model of Problem Setup~\ref{ass:dt-general} 
evaluates $\bigo{\samplesize(KN + mKs)}$ features overall.
\end{theorem}

\begin{proof}
The naive approach follows from a simple observation. If we knew beforehand the set of features that are involved in a tree $\g^{(j)}$, then in order to learn the tree, at any given node we require the learner to evaluate $\gain$ only over these features to determine the best split at that node. Thus, our protocol will just maintain the set of features present in any tree learned from scratch so far, so that \LFD{} can use these as ``metafeatures'' to carry out its evaluations limited to these features. Then, any target that can be represented using metafeatures $f \in \F$ that have been seen before in some other target, will be learned using our metafeatures. In other words when \LFD{} fails, the target is guaranteed to contain an ``unseen'' metafeature from $\F$. Thus, we will learn targets from scratch at most $|\F| = K$ times. Since each metafeature in $\F$ has at most $s$ distinct features, we will have to evaluate only at most $Ks$ features when not learning from scratch. 
\end{proof}

We now present the pseudocode for the different subroutines described informally in our discussion.

\begin{algorithm}[H]
\caption{$\affix(f,\nodeA,f')$: Affix $f'$ to $f$ at empty leaf node $\nodeA$ in $f$}
\begin{algorithmic}[1]
\STATE \textbf{Input:} Incomplete decision trees $f, f'$, empty leaf node $\nodeA$ in $f$
\STATE Assign to $\Var(\nodeA)$ the root variable of $f'$.
\STATE Create descendants nodes of $\nodeA$ and assign variables to them such that the tree rooted at $\nodeA$ is identical to $f'$.
\end{algorithmic}
\end{algorithm}

\begin{algorithm}[H]
\caption{$\lab(f,\nodeA,l)$: Assign $l$ to $\nodeA$ in $f$}
\begin{algorithmic}[1]
\STATE \textbf{Input:} Incomplete decision tree $f$, empty leaf node $\nodeA$ in $f$, label $l \in \{+,-\}$
\STATE Assign to leaf node $\nodeA$ the label $l$.
\end{algorithmic}
\end{algorithm}

\begin{algorithm}[H]
\caption{$\conflict(f,\nodeC,\nodeA, f')$ and $\induce(f,\nodeC,\nodeA, f')$}
\begin{algorithmic}[1]
\STATE \textbf{Input:} Incomplete decision trees $f, f'$, node $\nodeC$ in $f$, node $\nodeA$ that is a descendant of $\nodeC$ or equal to $\nodeC$ itself.
\STATE Let $\mathcal{V}$ be the set of nodes in $f$ that are ancestors of $\nodeA$ but not of $\nodeC$.
\STATE Map $\nodeC$ in $f$ to the root node of $f'$.
\STATE Similarly map all descendant nodes of $\nodeC$ from $\mathcal{V}$ to the nodes in the corresponding path in $f'$. 
\STATE \textbf{Output of $\conflict(f,\nodeC,\nodeA, f')$}: If there are two internal nodes $\nodeB \in f$ and $\nodeB' \in f'$ mapped to each other but $\nodeB \in \mathcal{V}$, $\Var(\nodeB) \neq \Var(\nodeB')$, output true. Else output false. 
\STATE \textbf{Output of $\induce(f,\nodeC,\nodeA, f')$}: Let $\nodeA'$ be the node from $f'$ mapped to $\nodeA$. Output $\Var(\nodeB')$. 
\end{algorithmic}
\end{algorithm}

We now prove our result for the semi-adversarial model, where in any given target, each $f \in \F$ has at least a $p_{min}$ probability of being the topmost metafeature.

 \begin{theorem} [\textbf{Lifelong learning of decision trees in semi-adversarial model}]
\label{thm:semi-random-dt}
There exists a lifelong learning protocol for decision trees that evaluates $\bigo{{\frac{1}{p_{\min}} \log \frac{K}{\delta}} \cdot N + m(K+d)}$ features overall 
%with a $\left(K+d, \bigo{\frac{1}{p_{\min}} \log \frac{K}{\delta}} \right)$-\efficiency{} (with high probability $1-\delta$) 
in a semi-adversarial model where each element of $\F$ has at least a $p_{\min}$ probability of being the topmost element of any target. The protocol learns only the first $\bigo{{\frac{1}{p_{\min}} \log \frac{K}{\delta}}}$ targets from scratch, adds them to $\tildeF$ and then uses \LFD{} Algorithm~\ref{alg:lfd-dt} to learn all the subsequent targets from $\tildeF$.
 \end{theorem}
 
 Recall that direct application of Lemma~\ref{lem:lfd-dt-feature-efficiency} implies that we will learn the subsequent targets examining  $\bigo{\frac{1}{p_{\min}}\log K+d}$ features per example. However, a more careful analysis making use of the fact that each element in $\tildeF$ is in fact from $\dtspace(\F)$ shows that we will examine only $\bigo{K+d}$ features per example. Note that this is an improvement because $\frac{1}{p_{\min}}\log K \geq K \log K$.

 \begin{proof} Consider the protocol from Theorem~\ref{thm:semi-random-dt} that learns the first $\bigo{\frac{1}{p_{\min}} \log \frac{K}{\delta}}$ targets from scratch, and adds them all to $\tildeF$. Then
with probability at least $1-\delta$, each metafeature from $\F$ will be at the top of some metafeature from $\tildeF$.  That is, $\dtspace(\F) \subseteq \dtspace(\Pref(\tildeF))$. Then, from Theorem~\ref{thm:lfd-dt} clearly Algorithm~\ref{alg:lfd-dt} can learn any future target from $\dtspace(\F)$ as the target will also lie in $ \dtspace(\Pref(\tildeF))$. Now, by a direct application of Theorem~\ref{thm:lfd-dt} this means we evaluate $\bigo{\frac{1}{p_{\min}} \log \frac{K}{\delta} + d}$ features per example.  

However, we can prove a tighter bound of $O(K+d)$ by following the proof technique for Lemma~\ref{lem:lfd-dt-feature-efficiency} but using to our advantage the fact that the metafeatures in $\tildeF$ are not arbitrary trees, but in fact members of $\dtspace(\F)$.  First of all, observe that the number of type A costs along any path is in fact $K$ and not $|\tildeF|$ because the metafeatures in $\tildeF$ can have only one of at most $K$ variables at its root.  Now, for the first case within type B, we will pay a cost of $d$ as before. However, for the second case, observe that any variable that is induced at $\nodeA$ by a metafeature $\tilde{f} \in \tildeF$, is in effect induced by a metafeature $f \in \F$. That is, when we compute $\induce(\tilde{g}, \nodeC_{\tilde{f}}, \nodeA, \tilde{f})$ for some metafeature $\tilde{f} \in \tildeF$, we effectively compute $\induce(\tilde{g}, \nodeC_{f}, \nodeA, f)$ for some metafeature $f \in \F$. Similarly we can argue that whenever we make $k_{\nodeA}$ distinct feature queries at a particular node $\nodeA$ during the algorithm, for all nodes beyond $\nodeA$ in that path, we effectively eliminate queries arising from $k_{\nodeA}-1$ metafeatures from $\F$ (and not $\tildeF$ as before). This will result in a total cost of $|\F| = K$ for this case. 
\end{proof}

%\dt*

\subsection{More Lifelong Learning Models for Decision Trees}
\label{app:relaxed}
\subsubsection{Decision Trees with Anchor Variables}
\label{app:root-anchor}

In this section, we consider a lifelong learning model of decision trees that assumes a more structured representation where each metafeature in $\tildeF$ has a variable at its root that does not occur in any other metafeature.

\begin{assumption} \label{ass:dt-anchor} Besides the assumptions in Problem Setup~\ref{ass:dt-general}, we assume that for each metafeature $f_i \in \F$ there exists a unique anchor variable $a_i \in [N]$ that occurs only at the root node of $f_i$
and not in any other node of $f_i$ or any other metafeature of $\F$.  \end{assumption}

In this setup, we again use \LFD{} Algorithm~\ref{alg:lfd-dt}. However, for \ID{}, we modify Algorithm~\ref{alg:id-dt} slightly. More specifically, after identifying a path in $g$ that was learned incorrectly using $\tildeF$, we pick exactly one subtree from this path and add it to $\tildeF$ (instead of all $d$ subtrees). We show that the total number of features evaluated reduces from a factor of $Kd$ to $K+d$.

\begin{theorem} 
\label{thm:dt-anchors}
In the model of Problem Setup~\ref{ass:dt-anchor}, the (\LFD{} Algorithm~\ref{alg:lfd-dt},\ID{} Algorithm~\ref{alg:id-dt-anchors})-protocol for decision trees  evaluates $\bigo{KN + m(K+d)}$ features overall. 
\end{theorem}

%Our idea is that when we fail on a target, we identify a specific subtree in the target that we add to our representation $\tildeF$, which helps us progress a step further in making $\tildeF$ more useful. 

\begin{proof} 
Like we did in the proof for Theorem~\ref{thm:dt}, we will show by induction that if $k$ targets have been learned from scratch, then there exists a set of $k$ true metafeatures $\F' \subseteq \F$ such that each metafeature $f \in \F'$ is the prefix of some metafeature in $\tildeF$. Then as we saw earlier, after learning $K$ trees from scratch, we can show that learning using $\tildeF$ will never fail. To prove our induction hypothesis, we claim that in any incorrectly learned path of $g$, the topmost node (say $\nodeA$) that conflicts with the incorrect output $\tilde{g}$ has to contain an anchor variable that is not at the root of any metafeature in  $\F'$.  This would mean that when we place the subtree rooted at $\nodeA$ in $\tildeF$, we are adding a tree whose suffix is an $f \in \F$ that does not belong to $\F'$. Essentially, we strictly increase the number of learned metafeatures by $1$ for every failure of \LFD{}.

Now we need to prove that $\nodeA$, the topmost conflicting node in some path of $g$ indeed contains an anchor variable that is not at the root of any metafeature from $\F'$. Let $\nodeA'$ be the corresponding node in $\tilde{g}$. This means that for all ancestors of $\nodeA'$, we assigned the correct variable, but something went wrong in $\nodeA'$ and hence $\Var(\nodeA) \neq \Var(\nodeA')$.  

Now, if $\Var(\nodeA)$ was an anchor variable, but one that occurs already at the root of some $f \in \F'$, we will certainly assign $\Var(\nodeA)$ to $\nodeA'$ which is a contradiction. On the other hand, consider the case in which $\Var(\nodeA)$ is a non-anchor variable. Then $\nodeA$ corresponds to a metafeature $f$ that occurs in $g$ and furthermore, the anchor variable in $f$ is in one of $\nodeA$'s ancestors, say $\nodeC_{f}$. In other words, $\conflict(g, \nodeC_{f}, \nodeA, f)$ is false and $\induce(g, \nodeC_{f}, \nodeA, f) = \Var(\nodeA)$. Note that by definition of $\nodeC'$, the corresponding node of $\nodeC_{f}$ in $\tilde{g}$, say $\nodeC_{f}'$, has been assigned the correct anchor variable $\Var(\nodeC_{f})$. Note that in the algorithm this assignment would have corresponded to a particular metafeature $\tilde{f} \in \tildeF$ and a node $\nodeC_{\tilde{f}}'$ in $\tilde{g}$ such that $\conflict(\tilde{g}, \nodeC_{\tilde{f}}', \nodeC_{f}', \tilde{f})$ is false and $\induce(\tilde{g}, \nodeC_{\tilde{f}}', \nodeC_{f}', \tilde{f}) = \Var(\nodeC_{f})$. By the run of Algorithm~\ref{alg:id-dt-anchors}, we have that in $\tilde{f}$, if the anchor variable of $f$ exists then $f$ exists as a whole too. More formally, this translates to $\conflict(\tilde{g}, \nodeC_{\tilde{f}}', \nodeA', \tilde{f})$ being false and $\induce(\tilde{g}, \nodeC_{\tilde{f}}', \nodeA', \tilde{f}) = \Var(\nodeA)$. This means that we will indeed assign $ \Var(\nodeA)$ to $\nodeA'$ which is a contradiction.
Thus, $\nodeA$ can only contain an anchor variable not already the root of any element in $\F'$. 
\end{proof}

\begin{algorithm}[H]
\caption{\ID{} - Decision Trees with anchor variables at the root}
\label{alg:id-dt-anchors}
\begin{algorithmic}[1]
\STATE Input: Old representation $\Fold$ and a tree $\g \in \dtspace(\F)$ learned from scratch and the incorrect tree $\tilde{\g}$ learned using $\Fold$.
\STATE $\tildeF \gets \Fold$

\STATE Identify a path starting at the root of $\tilde{g}$ such that the corresponding path in $g$ is shorter. 
\STATE Identify the topmost node in this path in $g$ which conflicts with the corresponding node in $\tilde{g}$.
\STATE Add the subtree in $g$ rooted at this node to $\tildeF$.
\STATE Return $\tildeF$
\end{algorithmic}
\end{algorithm}

\subsubsection{Sparse Decision Trees with Overcomplete Representations}

In this section, we consider another model wherein we assume that we have a very large metafeature set (of cardinality greater than $N$) and that each decision tree is constructed in a semi-adversarial manner. Our model, in some sense, is intended to capture noise. In particular, consider a metafeature set that is generated from the much smaller metafeature set from  Section~\ref{app:root-anchor} by creating many noisy duplicate copies of each metafeature. The noisy duplicates preserve the structure and the root variable of the original metafeature but may have different variables located in its non-root nodes. Clearly, this metafeature set affords a much larger representation which captures slight deviations from a rigid pattern. First observe that the ``anchor'' variables are no longer unique to a single metafeature, but are common to multiple metafeatures that however have the same structure. Now, we assume that each anchor variable has at least a $p_{\min}$ probability of being the root variable in any target. Note that this is not as strong an assumption as the previous semi-adversarial model because this allows for the case where some metafeatures do not occur in the top of the model at all.  Finally, we assume that our targets require only {\em sparse representations} in that along any path down the target, at most $t$ metafeatures from $\F$ have been affixed.  Below, we state our model formally.

\begin{assumption} \label{ass:dt-overcomplete} Besides every assumption in Problem Setup~\ref{ass:dt-general} except the metafeature assumption, we assume the following:
\begin{itemize}
\item \textbf{Metafeatures}: We assume that the metafeature set $\F = \F_1 \cup \F_2 \hdots \cup \F_{K_2}$ where each $\F_k$ consists of at most $K_1$ metafeatures of the same tree structure and the same root anchor variable $a_k$. This root anchor variable does not occur anywhere else in $\F$. 
\item \textbf{Semi-adversary}: Each anchor $a_k$ has at least a $p_{\min}$ probability of being the root metafeature in any target $g^{(j)}$. 
\item \textbf{Sparsity}: Any target $g^{(j)}$ can be constructed using $\F$ in a manner that uses at most $t$ metafeatures down any path from the root to a leaf in $g^{(j)}$. Typically $t \ll K_2$.
\end{itemize}
\end{assumption}

Observe that the metafeature set is of cardinality at most $K_1 K_2$. We now present a lifelong learning protocol that learns at most $K_1K_2+\bigo{\frac{1}{p_{\min}} \log \frac{K_2}{\delta}}$ targets from scratch, and learns the rest examining only  $\bigo{t K_1 + K_2 + d}$ features per example. Thus, given a constant sparsity parameter $t$, to ensure that we evaluate $o(mN)$ features, we can allow dictionaries of cardinality $\K_1\K_2 = o(\N^2)$. We now state our result formally. The idea is that we first learn a few targets from scratch and identify the anchors. Then, we partition any target that \LFD{} fails on into trees rooted at one of these anchors and add these trees as metafeatures hoping that we add at least one new metafeature from $\F$ to our representation.

\begin{theorem} 
\label{thm:semirandom-dt}
There exists a lifelong learning protocol for decision trees in the model of Problem Setup~\ref{ass:dt-overcomplete} that evaluates $\bigo{\left( {K_1K_2 + \bigo{\frac{1}{p_{\min}} \log \frac{K_2}{\delta}}} \right) N + m(K_1t + K_2+d)}$ features overall.
%achieves a feature-efficiency of $\bigo{K_1 t + K_2}$ while learning at most ${K_1K_2 + \bigo{\frac{1}{p_{\min}} \log \frac{K_2}{\delta}}}$ targets from scratch. 
The algorithm first learns $\bigo{\frac{1}{p_{\min}} \log \frac{K_2}{\delta}}$ targets from scratch to identify the $K_2$ anchor variables. The algorithm then uses \ID{} Algorithm~\ref{alg:id-semirandom-dt} and \LFD{} Algorithm~\ref{alg:lfd-dt}.
\end{theorem}

\begin{proof}
Let $\I_{\F}$ be the set of $K_2$ anchor variables. Under our assumptions, with high probability each of them will be the root of one of the first $\bigo{\frac{1}{p_{\min}} \log \frac{K_2}{\delta}}$ targets, and since no other variable can be a root of any target, we will identify them completely and correctly.

In any future tree that \LFD{} fails on, we learn the tree from scratch and partition the tree into metafeatures based on $\I_{\F}$ and them to $\tildeF$. We claim that $\tildeF \subseteq \F$ at any point of time and its cardinality strictly increases with each failure of \LFD{}. Then with $K_1K_2$ failures of \LFD{}, we will have $\tildeF=\F$, after which we will not see any failure. Assume this is true at some point of the run. When \LFD{} fails on a new target $g$, it means that $g \notin \dtspace(\Pref(\tildeF))$. However, since $g \in \dtspace(\F)$, this implies that $g$ is constructed using at least one metafeature $f \in \F - \tildeF$. Now observe that we would have identified the root and leaves of $f$ in $g$ correctly (because we would have identified all anchors in $g$ correctly). Then, we would have added $f$ to $\tildeF$, thereby satisfying our induction hypothesis.

By a direct application of Lemma~\ref{lem:lfd-dt-feature-efficiency} on the representation $\tildeF$, we get that we examine $\bigo{K_1K_2 + d}$ features per example which is uninteresting. However, we can tweak the argument we had for its proof for this case. First of all, we will have only $K_2$ type A costs (i.e., feature examinations) and not $K_1K_2$. Then, for type B costs, in sub-case $a$, we will have a cost of $d$ as before. For sub-case $b$, the cost was equal to the number of metafeatures in $\tildeF$, which would equal $K_1 K_2$ in this case. However, note that these costs correspond to $\induce(\tilde{g}, \nodeC_{\tilde{f}}, \nodeA, \tilde{f})$ for different $\tilde{f}$ such that $\nodeC_{\tilde{f}}$ contains the anchor variable in $\tilde{f}$. In total, we know that there are only at most $t$ anchor variables along a particular path, and hence only $K_1 t$ different metafeatures effectively result in some feature costs of this type. Hence, by restricting our analysis to only these metafeatures, we can show that the feature cost is proportional to $K_1 t$ and not to $K_1 K_2$.
In total, this would amount to a cost of $\bigo{K_1t + K_2 + d}$
\end{proof}

\begin{algorithm}[H]
\caption{\ID{} - Decision Trees with a Sparse but Overcomplete Representation}
\label{alg:id-semirandom-dt}
\begin{algorithmic}[1]
\STATE Input: Old representation $\Fold$, $\I_{\F}$ the set of anchor variables, and a tree $\g \in \dtspace(\F)$ learned from scratch.
\STATE $\tildeF \gets \Fold$
\STATE Identify the locations of variables from $\I_{\F}$ in $\g$ and partition $\g$ into trees rooted at one of these variables each. Add each tree to $\tildeF$.
\STATE Return $\tildeF$
\end{algorithmic}
\end{algorithm}

%% file: pm-appendix.tex
%In this section, we elaborate on our results for lifelong learning of monomials and polynomials.

 In Appendix~\ref{app:pm-scratch}, we present a simple algorithm for learning monomials exactly from scratch under some assumptions.  Then in Appendix~\ref{app:monomials-proofs}, we present our baseline lifelong learning algorithm for monomials. We also present Lemma~\ref{lem:pit} which we used to show that it is sufficient to check our prediction on a single randomly drawn example to verify whether the monomial we learned is correct.
%  proofs from Section~\ref{sec:monomials}. %inally, in Appendix~\ref{app:poly} we provide a detail discussion of sparse polynomials.

\subsection{Learning Monomials from Scratch}
\label{app:pm-scratch}

Recall that for any input  $\x = (x_1, x_2, \hdots x_\N) \in \mathbb{R}^N$, we denote the output of a $d$-degree target monomial $\G = (g_1, g_2, \hdots, g_N)$ by the function
 $P_{\G}(\x) = x_1^{\g_1} x_2^{\g_2} \hdots x_N^{g_N}$ where $\g_i \in \mathbb{N} \cup \{ 0 \}$ and the {degree} $\sum_i \g_i \leq d$. We denote the unknown metafeature set $\F = \{f_1, f_2, \hdots \}$ also as a matrix where column $i$ is  $f_i$. Therefore, saying that $\G$ can be expressed using $\F$ is equivalent to saying $\G$ lies in the column space of $\F$ denoted by $\C(\F)$. Then for any $k-$rank ($k \leq K$), $\N \times k$ matrix $\tildeF$ and for any $\G \in \C(\tildeF)$, we define $\w_\tildeF(\G) \in \mathbb{R}^k$ to denote the unique vector of column weights such that $\tildeF \w_\tildeF(\G) = \G$.

For each monomial target, we assumed that $\distr^{(j)}$ is a product distribution i.e., the features are independent.  We now state some specific assumptions about $\distr^{(j)}$.  In particular, we assume that the variance of each variable $x_i$ is not too small. The rationale is that if the variance was very small (in the extreme case, imagine $x_i$ being a constant), the factor $x_i^{g_i}$ would essentially be a constant factor in the monomial target. While it may be possible to design a more careful learning algorithm that can extract these nearly constant factors, that is beyond the scope of our discussion.  

Secondly, we assume that the probability density function is finite at every point i.e., the probability distribution is not too concentrated at any point. We will use this assumption to apply Lemma~\ref{lem:pit} when we draw a single sample to verify whether the monomial we have learned matches the true monomial.  

Finally, we assume that the support of $x_i$ is $[1,2]$. While the upper bound of $2$ is to simplify our discussion, the lower bound is to avoid dealing with values of $x_i$ that are close to zero. This is essential because as we will see later, we will  deal with logarithmic values of $x_i$ in the learning process. We now state our assumptions formally.

\begin{assumption2} \label{ass:monomial-distr} Each $\distr^{(j)}$ is a product distribution. Let $\distr^{(j)} = \mu_1^{(j)} \times \cdots \times \mu_{N}^{(j)}$. We assume that for all features $i$:
\begin{itemize}
\item \textbf{Minimum variance} $Var_{\mu_i^{(j)}}(\log x_i) \geq c$.
\item \textbf{Bounded probability density} $\forall x_i \in \mathbb{R}$, $\mu_i^{(j)}(x_i) \in \mathbb{R}$. % $Pr(x_i) \leq \frac{\delta}{md}$.
\item \textbf{Bounded support} The support of $\mu_i^{(j)}$ is $[1,2]$.
%$\forall x_i \in \mathbb{R}$, 
%$Pr_{x_i \sim \mu_i}(|x_i| \geq B) = 0.$
%$Var(\log x_i | E_{N}(\x, \frac{1}{16bN})) \geq c$.
\end{itemize}
\end{assumption2}

We now present our simple poly-time technique for learning monomials from scratch with polynomially many samples. Recall that the output of the monomial $\G$ on an input $\x$ is denoted by $P_{\G}(\x)$. Let us denote the logarithm of this output $\log |P_{\G}|$ by $Q_{\G}$. Observe that learning $\G$ is equivalent to learning the coefficients of the `linear' function $Q_{\G}$. To see how this can be done, we will define a notion of correlation/inner product of two functions  $h(\x)$ and $h'(\x)$:
\[
\langle  h(\x), h'(\x) \rangle \triangleq \E[h(\x)h'(\x)].
\] 

Then, we claim that $g_i$ can be expressed as the following inner product. % (the proof for which can be found in Appendix~\ref{app:pm}).
\begin{lemma}
\label{lem:deg} 
%Define the function \[
%Q_{\G}(\x) = = \sum\limits_{i=1}^{N} \g_{i} \log |x_i| .
% \]
% Then, 
\[
\frac{\left\langle  Q_{\G}(\x), \log(x_i) - \E[\log(x_i)]\right\rangle }{\E[\log^2 x_i ] - \E^2[\log x_i]} = \g_i
\]
\end{lemma}

\begin{proof}
Since $x_i$ is picked independent of the other variables, so is the random variable $(\log x_i - \E[\log(x_i)])$. Thus, when $j \neq i$
\[
\E[\log x_j (\log x_i - \E[\log(x_i)])] = \E[\log x_j] \times \E[\log x_i - \E[\log(x_i)]] = 0
\]
However,
\[
\E[\log x_i (\log x_i - \E[\log(x_i)])] = \E[\log^2 x_i] - \E^2[\log x_i]
\]

 Then, the claim follows from our definition of $Q_{\G}$.
\end{proof}

Observe that using the above fact, we can calculate $g_i$ for each $i \in [N]$ exactly if we were provided the exact values of each correlation term in the equality. However, the best we can hope for is to approximate these terms using sufficiently many samples.  Fortunately, we can actually approximate each of these correlation terms to a small constant error such that these errors together imply a constant error smaller than $1/2$ in estimating $g_i$. Then we can round off our estimate to the closest natural number to find the exact value of $g_i$. We now summarize our simple algorithm for learning a monomial from scratch, and then prove our polynomial sample complexity bound.

\begin{algorithm}
\caption{Learning a monomial from scratch}
\label{alg:monomials-scratch}
\begin{algorithmic}[1]
\STATE Input: Distribution $\distr$ over $\mathbb{R}^{N}$
\STATE  Draw $\samplesize$ samples $(\x, P_{\G}(\x))$ from $\distr$ and query \text{all} the features  on all samples.
\FOR{$i=1,2,\hdots N$}
	\STATE Estimate $\E[\log^2 x_i ]$, $\E[\log^2 x_i ]-\E^2[\log x_i]$, and $\left\langle  Q_{\G}(\x), \log(x_i) - \E[\log(x_i)]\right\rangle $ empirically.
	\STATE Round off \[
\frac{\left\langle  Q_{\G}(\x), \log(x_i) - \E[\log(x_i)]\right\rangle }{\E[\log^2 x_i ] - \E^2[\log x_i]}
\]
to estimate $g_i$.
\ENDFOR
\STATE Return $\tilde{\G}$
\end{algorithmic}
\end{algorithm}

Clearly the above algorithm has polynomial running time and sample complexity as long as $\samplesize$ is polynomial. The crucial guarantee we need now is that polynomially many samples are sufficient to estimate each $g_i$ exactly, which we show in Theorem~\ref{thm:pm-scratch}. 
We first begin by bounding the error in estimating the numerator $\left\langle  Q_{\G}(\x), \log(x_i) - \E[\log(x_i)]\right\rangle$ in Lemma~\ref{lem:eps-1-3}. Then, in Lemma~\ref{lem:g-hat-g} we show how this error and the error in the denominator terms, add up to result in an error of at most $1/2$ in estimating $g_i$.  Using these, we prove in Theorem~\ref{thm:pm-scratch} that the algorithm estimates each power exactly. In the following notation we will use $\tilde{\E}$ to denote the empirical estimate of an expected value. 

\begin{lemma}
\label{lem:eps-1-3}
Using a sample set $\sample$ of size 
$\bigo{ \frac{d}{\epsilon_3^2} \log\frac{1}{\delta'}}$, for a given $i \in [N]$,  if $|\tilde{\E}[\log x_i] - \E[\log x_i]| \leq \epsilon_1$, then we can guarantee that  
\[Pr\left[
\left\lvert \frac{1}{|S|} \sum_{\x \in S} Q_{\G}(\x)(\log(x_i) - \tilde{\E}[\log(x_i)]) -  \langle  Q_{\G}(\x), \log(x_i) - \E[\log(x_i)]\rangle\right\rvert \leq d \epsilon_1  + \epsilon_3
\right] = \bigo{\delta'} \].
%$\bigo{ \frac{d}{\epsilon_3^2} \left( 2\log^2 B -  \min \left( 0, \log \frac{1}{16 Nb} \log \frac{1}{16db}\right)  - \log B\min \left( 0, \log\frac{1}{16db} \right) \right) \log\frac{1}{\delta_1}}$

\end{lemma}

\begin{proof} Consider the random variable $Q_{\G}(\x) \cdot ( \log(x_i) - \tilde{\E}[\log(x_i)])$. It is easy to show that $Q_{\G}(\x) \log(x_i)  \in [0, d]$ with the extreme values attained at $\x = (2, 2, \hdots)$ and $\x= (1,1, \hdots)$. Then, $Q_{\G}(\x)  \E[\log(x_i)] \in [0,d]$. Thus, the random variable $Q_{\G}(\x) \cdot ( \log(x_i) - \tilde{\E}[\log(x_i)])$ lies in a range of size $2d$. Then, by Chernoff bounds, we can show that
\[Pr\left[
\left\lvert \frac{1}{|\sample|} \sum_{\x \in \sample} Q_{\G}(\x)(\log(x_i) -\tilde{{\E}}[\log(x_i)]) -  \langle  Q_{\G}(\x), \log(x_i) - \tilde{\E}[\log(x_i)]\rangle\right\rvert \leq \epsilon_3
\right]= \bigo{\delta'} \] from which the above claim follows because the absolute difference between $  \langle  Q_{\G}(\x), \log(x_i) - \E[\log(x_i)]$ and $  \langle  Q_{\G}(\x), \log(x_i) - \tilde{\E}[\log(x_i)]$ is at most  $\left\lvert \max_{\x}Q_{\G}(\x)\cdot (\E[\log(x_i)]) - \tilde{\E}[\log(x_i)])) \right\rvert \leq d  \epsilon_1$ (because the first term is at most $d$ and the next is at most $\epsilon_1$).
\end{proof}

\begin{lemma}
\label{lem:g-hat-g}
Using a sample set $\sample$ of size  $ \bigo{ \frac{d}{\left(\min (\frac{c^2}{d}, \frac{c}{d}, 1)\right)^2}\log \frac{1}{\delta'} }$ with a high probability  of $1 - \delta'$ for a given $i \in [N]$ we can learn  $\tilde{g}_i$ such that $|\tilde{\g_i} - \g_i| \leq \frac{1}{2}$.
\end{lemma}

\begin{proof}
Let $\epsilon_1$ and $\epsilon_3$ be as defined in Lemma~\ref{lem:eps-1-3}. Additionally let $|\tilde{\E}[\log^2 x_i] - \E[\log^2 x_i]| \leq \epsilon_2$. From the previous results and from Chernoff bounds, we have that $\epsilon_1, \epsilon_2, \epsilon_3$ are all $\bigo{\min (\frac{c^2}{d}, \frac{c}{d}, 1)}$ given the size of $\sample$. We now have a fractional expression on the right hand side of the equation in Lemma~\ref{lem:deg} for which we can derive the error in estimating the numerator and the denominator individually. We need to show that the overall error in estimating the fraction is $1/2$ i.e., $O(1)$. Now, the error in estimating some fraction $\frac{G}{H}$ using $\frac{\tilde{G}}{\tilde{H}}$ given that $|G - \tilde{G}| \leq \epsilon_G$ and $|H - \tilde{H}| \leq \epsilon_H$  can be upper bounded by:
\[
\begin{array}{rcl}
\left|\frac{G \pm \epsilon_G}{H \pm \epsilon_H} -\frac{G}{H} \right|& = & \left|\frac{\epsilon_G}{H} \pm \frac{G\epsilon_H}{(H-\epsilon_H) H}\right| \\
& \leq &  \frac{\epsilon_G}{\min H} + \frac{(\max G+ \epsilon_G) \epsilon_H}{(\min H - \epsilon_H) \min H}
\end{array}
\]
In our case, we have 
$H = \E[\log^2 x_i] - \E^2[\log x_i]$ 
and $G = \langle Q_{\G}(\x), \log(x_i) - \E[\log(x_i)]\rangle$, $\min H = c$ 
and $\max G = d$.
Also, $\epsilon_G = \epsilon_1 d + \epsilon_3 $ and $\epsilon_H \leq  \epsilon_2 + 2 \epsilon_1 + \epsilon_1^2$. The latter inequality follows from the fact that the error in estimating $\E[\log^2 x_i]$ is $\epsilon_2$ and the error in estimating $ \E^2[\log x_i]$ is at most $(\E[\log x_i] + \epsilon_1)^2 - \E^2[\log x_i] \leq \epsilon_1 (2\E[\log x_i] + \epsilon_1) \leq \epsilon_1 (2 + \epsilon_1)$. By a simple calculation, it can be verified that this results in a total error of $\bigo{1}$ in estimating $g_i$.

\end{proof}

\begin{theorem}
\label{thm:pm-scratch}
Algorithm~\ref{alg:monomials-scratch} exactly learns a target $\G$ from scratch with high probability $ 1 - \bigo{\frac{\delta}{K}}$ with $\samplesize = \bigo{ \frac{d}{\left(\min (\frac{c^2}{d}, \frac{c}{d}, 1)\right)^2}\log \frac{Nm}{\delta} }$ samples.
\end{theorem}

\begin{proof}
From Lemma~\ref{lem:g-hat-g} we have that each $\g_i$ is accurately estimated with probability at least $1 - \bigo{ \frac{\delta}{Nm}}$. By a union bound, $\G$ is accurately estimated with probability at least $1 - \bigo{\frac{\delta}{m}}$.
\end{proof}

We note that it is easy to refine our application of union bounds to use slightly fewer samples than in the bound of Theorem~\ref{thm:pm-scratch}. In particular, it is possible bring the $\log Nm$ factor down to $\log NK$ while learning from scratch, and to $\log Km$ on all other targets.

\subsection{Naive Lifelong Learning of Monomials}
\label{app:monomials-proofs}
We  present our straightforward approach for lifelong learning of monomials which  merely keeps a record of features that have been seen in earlier targets.

\begin{theorem}[\textbf{Naive lifelong learning of monomials}]
\label{thm:naive-monomials}
In the model of Problem Setup~\ref{ass:monomial-general}, there exists a naive algorithm for lifelong learning of monomials that evaluates $\bigo{S(KN + mKd)}$ features overall. % $O(Kd,K)$-\efficient{}. % learns at most $K$ targets from scratch while achieving a feature-efficiency of $Kd$.
\end{theorem}

\begin{proof}(Sketch)
We use \ID{} Algorithm~\ref{alg:id-prod-mon} that essentially stores the list of targets that have been learned from scratch as the columns of the matrix $\tildeF$. Now, consider the set of  features that have been ``seen'' so far i.e., these correspond to rows in $\tildeF$ that have at least one non-zero entry. Then, for a new target $\G$, we define a \LFD{} algorithm that determines the powers of only these features. This can be done by evaluating only those features on the data set using the technique in Algorithm~\ref{alg:monomials-scratch}. The unseen features are assumed to have zero power.

Now, consider a new target $\G$ that is ``linearly dependent'' on the targets that have been learned so far i.e., $\G \in \C(\tildeF)$. In this case, the unseen features should have a zero exponent in $\G$ as it is zero in all earlier targets. Thus, our \LFD{} technique would not fail on such targets. Now, if $\G$ was linearly independent, it is possible that an unseen feature has a non-zero exponent in $\G$. To verify whether this is the case, we can draw a single sample and check whether our prediction matches the true output.  If this fails, we learn the target correctly from scratch and add it to $\tildeF$. 

Thus, since we add only linearly independent targets to $\tildeF$, in a manner similar to the proof of Theorem~\ref{thm:monomials}, we can show that \LFD{} will not fail more than $K$ times.  Our result follows from here because each of the targets that we learn from scratch have at most $d$ non-zero exponents. Then, in total we only have at most $Kd$ ``seen'' features i.e., features with non-zero powers that we always examine.
\end{proof}

\subsubsection{Monomial Identity Testing}

We show here that it is sufficient to draw a single example and check whether our prediction matches the true label in order to conclude whether the monomial that we learned is indeed the true monomial. Here, we make use of the condition that the probability distribution is smooth in that the probability density function at any value of a feature is finite.

 \begin{lemma}
\label{lem:pit}  If for every feature $i$, the marginal probability density function at $x_i$ is finite for all values of $x_i$
%Pr[x_i] = \bigo{ \frac{\delta}{md}}$
% (which is true according to Problem Setup~\ref{ass:monomial-distr}), 
then we have that for any $\G' \neq \G$, $Pr[P_{\G'}(\x) \neq P_{\G}(\x)] = 1$.% - \bigo{\frac{\delta}{m}}$.
\end{lemma}

\begin{proof}
We will prove by induction on $N' \leq N$ and $d' \leq d$  that for any polynomial $P'$ of degree $d'$ over $N'$ variables $Pr[P'(\x) = 0] = 0$. Then, we only need to plug in $P' = P_{\G} - P_{\G'}$ to complete the proof. 

For the base case assume the polynomial is only over one variable and any degree i.e., $N'=1$ and any $d' \leq d$.  Then the event $[P'(\x) = 0]$ corresponds to picking one of at most $d'$ zeroes of $P'$ from $\mathbb{R}$ (since $N'=1$), which amounts to a probability of $0$ according to the assumption on the probability density function.

Now assume for all $N' < N$ and $d' \leq d$, our induction hypothesis is true. The polynomial $P'$ can be expressed as a summation of terms in $x_1$:  $\sum_{i=0}^{k} P''_i(x_2, \hdots x_n) x_1^i$ where $k$ is the highest degree of $x_1$ and $P''_i$ is the coefficient of $x_1^i$.  Then, for a fixed value of $x_2, \hdots x_N$, $P'$ reduces to a polynomial of degree $k\leq d$  over one variable. Then, our induction assumption implies that conditioned on some arbitrary values of $x_2, \hdots, x_N$, the polynomial in $x_1$ attains zero with probability $0$ i.e., $Pr[P'(\x) =0 \, | \, x_2, \hdots x_N] = 0$. Then it follows that $Pr[P'(\x) = 0 ]=0$. 
\end{proof}

%% file: poly-appendix.tex
We now describe our straightforward lifelong learning approach for polynomials which remembers only the features that have been seen so far.

\begin{theorem}[\textbf{Naive lifelong learning of polynomials}]
\label{thm:naive-poly}
In the model of Problem Setup~\ref{ass:polynomial-general}, there exists a naive algorithm for lifelong learning of $t$-sparse polynomials that makes $\bigo{S(KN + mKd)}$ feature evaluations in total.
\end{theorem}

\begin{proof} (Sketch)
This approach is very similar to the naive approach for lifelong learning of monomials. We will use \ID{} Algorithm~\ref{alg:id-poly} which, as we know already, maintains a list of linearly independent monomial targets that have been seen in the polynomials learned from scratch so far. Now, for a new target $\setG$, we will perform the ``lexicographic search'' method from \cite{andoni} over only the features that have been seen i.e., during the search we skip features that correspond to an all zero row in $\tildeF$. Essentially, we assume that the unseen features do not occur in the target polynomial. We again check whether the polynomial computed this way is correct by verifying it on a single sample.

Using this approach we are guaranteed that if $\setG \subset \C(\tildeF) \times \mathbb{R}$, \LFD{} does not fail because such a target will not contain unseen features in any of its monomials. Then, we can use an argument similar to Theorem~\ref{thm:naive-monomials} and show by contradiction that \LFD{} can fail at most $K$ times, and hence evaluate only $Kd$ features per example. 
\end{proof}

%% file: main.bbl
\begin{thebibliography}{10}

\bibitem{BLworkshop10}
{International Conference on Machine Learning (ICML) Workshop} on budgeted
  learning, 2010.

\bibitem{andoni}
Alexandr Andoni, Rina Panigrahy, Gregory Valiant, and Li~Zhang.
\newblock Learning sparse polynomial functions.
\newblock In {\em Proceedings of the Annual {ACM-SIAM} Symposium on Discrete
  Algorithms {(SODA)}}, pages 500--510, 2014.

\bibitem{mtl}
Andreas Argyriou, Theodoros Evgeniou, and Massimiliano Pontil.
\newblock Multi-task feature learning.
\newblock In {\em Proceedings of the Annual Conference on Neural Information
  Processing Systems {(NIPS)}}, pages 41--48, 2006.

\bibitem{AroraGM14}
Sanjeev Arora, Rong Ge, and Ankur Moitra.
\newblock New algorithms for learning incoherent and overcomplete dictionaries.
\newblock In {\em Proceedings of the Conference on Learning Theory (COLT)},
  pages 779--806, 2014.

\bibitem{eff-rep}
Maria{-}Florina Balcan, Avrim Blum, and Santosh Vempala.
\newblock Efficient representations for lifelong learning and autoencoding.
\newblock In {\em Proceedings of the Conference on Learning Theory {(COLT)}},
  pages 191--210, 2015.

\bibitem{Baxter97}
Jonathan Baxter.
\newblock A bayesian/information theoretic model of learning to learn via
  multiple task sampling.
\newblock {\em Machine Learning}, 28(1):7--39, 1997.

\bibitem{baxter2}
Jonathan Baxter.
\newblock A model of inductive bias learning.
\newblock {\em Journal of Artificial Intelligence Research (JAIR)}, 12, 2000.

\bibitem{cart84}
L.~Breiman, J.~Friedman, R.~Olshen, and C.~Stone.
\newblock {\em {Classification and Regression Trees}}.
\newblock Wadsworth and Brooks, 1984.

\bibitem{ll-recom}
Hendrik Drachsler, Hans G.~K. Hummel, and Rob Koper.
\newblock Personal recommender systems for learners in lifelong learning
  networks: the requirements, techniques and model.
\newblock {\em International Journal of Learning Technology {(IJLT)}},
  3(4):404--423, 2008.

\bibitem{EA06}
Michael Elad and Michal Aharon.
\newblock Image denoising via learned dictionaries and sparse representation.
\newblock In {\em {IEEE} Computer Society Conference on Computer Vision and
  Pattern Recognition {(CVPR})}, pages 895--900, 2006.

\bibitem{KapoorG05}
Aloak Kapoor and Russell Greiner.
\newblock Learning and classifying under hard budgets.
\newblock In {\em Machine Learning: 16th European Conference on Machine
  Learning {(ECML)}}, pages 170--181, 2005.

\bibitem{KM}
Michael Kearns and Yishay Mansour.
\newblock On the boosting ability of top-down decision tree learning
  algorithms.
\newblock In {\em Proceedings of the Annual ACM Symposium on Theory of
  Computing (STOC)}, pages 459--468, 1996.

\bibitem{ll-a-star}
Sven Koenig, Maxim Likhachev, and David Furcy.
\newblock Lifelong planning {A}$_\ast$.
\newblock {\em Artificial Intelligence}, 155(1-2):93--146, 2004.

\bibitem{halnips12}
A.~Kumar and H.~Daume III.
\newblock Learning task grouping and overlap in multi-task learning.
\newblock In {\em {Proceedings of the International Conference on Machine
  Learning (ICML) }}, 2012.

\bibitem{LS00}
Michael~S Lewicki and Terrence~J Sejnowski.
\newblock Learning overcomplete representations.
\newblock {\em Neural computation}, 12(2):337--365, 2000.

\bibitem{LizotteMG03}
Daniel~J. Lizotte, Omid Madani, and Russell Greiner.
\newblock Budgeted learning of naive-bayes classifiers.
\newblock In {\em Proceedings of the Conference in Uncertainty in Artificial
  Intelligence (UAI)}, pages 378--385, 2003.

\bibitem{pm:13}
A.~Maurer and M.~Pontil.
\newblock Excess risk bounds for multitask learning with trace norm
  regularization.
\newblock In {\em Proceedings of the Annual Conference on Learning Theory
  {(COLT)}}, 2013.

\bibitem{fs}
Guillaume Obozinski and Ben Taskar.
\newblock Multi-task feature selection.
\newblock In {\em Workshop of Structural Knowledge Transfer for Machine
  Learning in the International Conference on Machine Learning (ICML)}, 2006.

\bibitem{PY10}
Sinno~Jialin Pan and Qiang Yang.
\newblock A survey on transfer learning.
\newblock {\em IEEE Transactions on Knowledge and Data Engineering},
  22(10):1345--1359, 2010.

\bibitem{wmv}
A.~Pentina and R.~Urner.
\newblock Lifelong learning with weighted majority votes.
\newblock In {\em Proceedings of the Annual Conference on Neural Information
  Processing Systems {(NIPS)}}, 2016.

\bibitem{Podgorelec2002}
Vili Podgorelec, Peter Kokol, Bruno Stiglic, and Ivan Rozman.
\newblock Decision trees: An overview and their use in medicine.
\newblock {\em Journal of Medical Systems}, 26(5):445--463, 2002.

\bibitem{Quinlan86}
J.~R. Quinlan.
\newblock Induction of decision trees.
\newblock {\em Machine Learning}, 1(1):81--106, March 1986.

\bibitem{RM08}
Lior Rokach and Oded Maimon.
\newblock {\em Data Mining with Decision Trees: Theory and Applications}.
\newblock World Scientific Publishing Co., Inc., 2008.

\bibitem{TP97}
S.~Thrun and L.Y. Pratt, editors.
\newblock {\em Learning To Learn}.
\newblock Kluwer Academic Publishers, 1997.

\bibitem{ll-robot}
Sebastian Thrun and Tom~M. Mitchell.
\newblock Lifelong robot learning.
\newblock {\em Robotics and Autonomous Systems}, 15(1-2):25--46, 1995.

\bibitem{top10dt}
Xindong Wu, Vipin Kumar, J.~Ross Quinlan, Joydeep Ghosh, Qiang Yang, Hiroshi
  Motoda, Geoffrey~J. McLachlan, Angus F.~M. Ng, Bing Liu, Philip~S. Yu,
  Zhi{-}Hua Zhou, Michael Steinbach, David~J. Hand, and Dan Steinberg.
\newblock Top 10 algorithms in data mining.
\newblock {\em Knowledge and Information Systems}, 14(1):1--37, 2008.

\bibitem{hastie}
Eric~R. Ziegel.
\newblock The elements of statistical learning.
\newblock {\em Technometrics}, 45(3):267--268, 2003.

\end{thebibliography}
